\documentclass[final,12pt]{colt2023} % Include author names

\title[Convergence of AdaGrad over Non-convex Landscape]{Convergence of AdaGrad for Non-convex Objectives: \\Simple Proofs and Relaxed Assumptions} 
\usepackage{hyperref}       % hyperlinks
\usepackage{url}            % simple URL typesetting
\usepackage{booktabs}       % professional-quality tables
\usepackage{amsfonts}       % blackboard math symbols
\usepackage{nicefrac}       % compact symbols for 1/2, etc.
\usepackage{microtype}      % microtypography
\usepackage{xcolor}         % colors
\usepackage{amsmath}
\usepackage{changepage}
\usepackage{pifont}
\usepackage{dsfont}

\usepackage[noend]{algpseudocode}
\usepackage{algorithmicx,algorithm}
\newtheorem{assumption}{Assumption}

\usepackage{algpseudocode}
\usepackage{bm}
\usepackage{graphicx, wrapfig}
\hypersetup{hidelinks}
\usepackage{color}
\usepackage{array}
\usepackage{multirow}
\usepackage{amssymb}

\newcommand{\bw}{\boldsymbol{w}}
\newcommand{\bx}{\boldsymbol{x}}

\newcommand{\bv}{\boldsymbol{v}}

\newcommand{\bnu}{\boldsymbol{\nu}}

\usepackage{float}
\usepackage{times}
\usepackage{wrapfig}
\usepackage{lipsum}
 \newcommand*{\QEDA}{\hfill\ensuremath{\blacksquare}}

\coltauthor{%
 \Name{Bohan Wang} \Email{bhwangfy@gmail.com}\\
 \addr University of Science and Technology of China
 \AND
 \Name{Huishuai Zhang} $^{*}$ \Email{huzhang@microsoft.com}\\
 \addr Microsoft Research Asia%
 \AND 
  \Name{Zhi-Ming Ma} \Email{mazm@amt.ac.cn}\\
 \addr Academy of Mathematics and Systems Science, Chinese Academy of Sciences
 \AND
  \Name{Wei Chen} \thanks{Corresponding Author}  \Email{chenwei2022@ict.ac.cn}\\
 \addr Institute of Computing Technology, Chinese Academy of Sciences
}

\begin{document}

\maketitle

% \begin{abstract}%
% Understanding AdaGrad is the initial point towards understanding adaptive optimizers over non-convex landscape. However, the convergence analysis of AdaGrad has not been well-established: all the existing works suffer either from unrealistic assumptions or from sophisticated proof and (as a result) a sub-optimal rate in the so-called over-parameterized regime. In this paper, we provide a refined analysis for AdaGrad. Specifically, we find a novel auxillary function $\xi$, based on which we are able to derive a simple convergence analysis of AdaGrad only assuming affine noise variance and bounded smoothness. Our result shows that in the over-parameterized regime, AdaGrad only needs $\mathcal{O}(\frac{1}{\varepsilon^2})$ to ensure that the gradient norm is smaller than $\varepsilon$, which matches that of SGD and is the first result of this kind. We then try to step beyond the 
% uniformly smooth landscape, and consider a simple yet realistic non-uniformly smooth condition, called $(L_0,L_1)$-smooth condition. Again based on the auxillary function $\xi$, we prove that AdaGrad(-Norm) succeeds to converge under such a condition while requiring the learning rate to be smaller than a threshold. We show such a requirement is essential through a counterexample, and demonstrate that AdaGrad loses the tuning-free ability when the smoothness is no longer bounded. Together, our analyses broaden the understanding of AdaGrad, and suggest that the auxillary function $\xi$ to be a powerful tool in the analyses of AdaGrad.
% \end{abstract}
\begin{abstract}%
We provide a simple convergence proof for AdaGrad optimizing non-convex objectives under only affine noise variance and bounded smoothness assumptions. The proof is essentially based on a novel auxiliary function $\xi$ that helps eliminate the complexity of handling the correlation between the numerator and denominator of AdaGrad's update. Leveraging simple proofs, we are able to obtain tighter results than existing results \citep{faw2022power} and extend the analysis to several new and important cases. Specifically, for the  over-parameterized regime, we show that AdaGrad  needs only $\mathcal{O}(\frac{1}{\varepsilon^2})$ iterations to ensure the gradient norm smaller than $\varepsilon$, which matches the rate of SGD and significantly tighter than existing rates $\mathcal{O}(\frac{1}{\varepsilon^4})$ for AdaGrad. We then discard the bounded smoothness assumption, and consider a more realistic assumption on smoothness called $(L_0,L_1)$-smooth condition, which allows local smoothness to grow with the gradient norm. Again based on the auxiliary function $\xi$, we prove that AdaGrad succeeds in converging under $(L_0,L_1)$-smooth condition as long as the learning rate is lower than a threshold. Interestingly, we further show that the requirement on learning rate under the $(L_0,L_1)$-smooth condition is necessary via proof by contradiction, in contrast with the case of uniform smoothness conditions where convergence is guaranteed regardless of learning rate choices. Together, our analyses broaden the understanding of AdaGrad, and demonstrate the power of the new auxiliary function in the investigations of AdaGrad.
\end{abstract}


\section{Introduction}
\label{sec: intro}
Adaptive optimizers have been a great success in deep learning. Compared to stochastic gradient descent (SGD), adaptive optimizers use the gradient information of iterations to dynamically adjust the learning rate, which is observed to converge much faster than SGD in a wide range of deep learning tasks \citep{vaswani2017attention,dosovitskiy2020image,yun2019graph}. Such a superiority has attracted numerous researchers to analyze the behavior of adaptive optimizers.

AdaGrad \citep{duchi2011adaptive} is among the earliest adaptive optimizers and enjoys favorable convergence rate for online convex optimization. Specifically, the design of AdaGrad is quite simple: it  tracks the gradient magnitudes of the past iterations and use its reciprocal to scale the current gradient. The pseudo-codes of the norm version of AdaGrad (i.e., AdaGrad-Norm) and  AdaGrad are  presented in Algorithm \ref{alg: adagrad_norm} and Algorithm \ref{alg: adagrad}, respectively.%, where $g_t$ is required that $\mathbb{E}[g_t|\sigma (g_1,\cdots,g_{t-1})]=\nabla F(\bw_t)$ (latter we denote $\mathcal{F}_t \triangleq \sigma (g_1,\cdots,g_{t-1})$ ).

Despite the popularity and the simplicity of AdaGrad, its theoretical analysis is not satisfactory when optimizing non-convex objectives. Specifically, until recently,  \cite{ward2020adagrad} analyze the convergence of AdaGrad-Norm and achieve $\mathcal{O}(\log T / \sqrt{T})$ rate. However, Their result is based on the assumption that the stochastic gradient $g_t$ is uniformly bounded across the iterations, which does not hold even for quadratic functions, let alone deep neural networks. In comparison, the analysis of SGD does not require such an assumption.

\begin{minipage}{0.46\textwidth}
\vspace{-2mm}
\begin{algorithm}[H]
    \centering
    \caption{AdaGrad-Norm}\label{alg: adagrad_norm}
    \hspace*{0.02in} {\bf Input:}
Objective function $f(\bw)$, learning rate $\eta>0$, initial point $\bw_{1} \in \mathbb{R}^d$, initial conditioner $\bnu_{1}\in \mathbb{R}^{+}$~~~~~~~~~~~~~~~~~~~~~~~~~~~~~~~~~
    \begin{algorithmic}[1]
       \State \textbf{For} $t=1\rightarrow \infty$:
\State ~~~~~Generate stochastic gradient $g_t$
\State ~~~~~Calculate $\bnu_{t}=\bnu_{t-1}+\Vert g_t\Vert^2$
\State ~~~~~Update $\bw_{t+1}= \bw_{t}-\eta \frac{1}{\sqrt{\bnu_{t}}}  g_{t} $
\State \textbf{EndFor}
    \end{algorithmic}
\end{algorithm}
\end{minipage}
\hfill
\begin{minipage}{0.46\textwidth}
\vspace{-2mm}
\begin{algorithm}[H]
    \centering
    \caption{AdaGrad}\label{alg: adagrad}
    \hspace*{0.02in} {\bf Input:}
Objective function $f(\bw)$, learning rate $\eta>0$, initial point $\bw_{1} \in \mathbb{R}^d$, initial conditioner $\bnu_{1}\in \mathbb{R}^{d,+}$~~~~~~~~~~~~~~~~~~~~~~~~~~~~~~~
    \begin{algorithmic}[1]
       \State \textbf{For} $t=1\rightarrow \infty$:
\State ~~~~~Generate stochastic gradient $g_t$
\State ~~~~~Calculate $\bnu_{t}=\bnu_{t-1}+ g_t^{\odot 2}$
\State ~~~~~Update $\bw_{t+1}= \bw_{t}-\eta \frac{1}{\sqrt{\bnu_{t}}} \odot g_{t} $
\State \textbf{EndFor}
    \end{algorithmic}
\end{algorithm}
\end{minipage}
\vspace{3mm}

A very recent exception  \citep{faw2022power} relaxes the assumptions and proves that AdaGrad-Norm converges by only assuming  uniformly bounded smoothness (c.f. our Assumption \ref{assum: smooth}) and affine noise variance (c.f. our Assumption \ref{assum: affine}), which matches the conditions of SGD. However, the proof in \citep{faw2022power} is rather complicated (around 30 pages), which is hard to understand the intuition behind and to extend to the analysis of other cases. Moreover, {  the convergence rate in \citep{faw2022power} does not get better when strong growth condition holds (i.e., our Assumption \ref{assum: affine} with $D_0=0$) while SGD does. We believe such a gap is vital as strong growth condition holds in over-parameterized models \citep{vaswani2019fast}, which are widely adopted in deep learning.}

% there is a gap between the rates of AdaGrad-Norm and SGD for over-parameterized scenarios, which are believed to be important cases in deep learning. 

%First, Assumption \ref{assum: affine} with $D_0=0$ is called the strong growth condition by existing literature \citep{xxx}, under which the averaged squared gradient $\frac{1}{T} \sum_{t=1}^T \Vert \nabla f(\bw_t) \Vert^2$ of SGD is shown to converge with rate $\mathcal{O}(\frac{1}{T})$, which is faster than the rate of SGD when $D_0\ne 0$. On the other hand, the rate of AdaGrad in \citep{faw2022power} is $\mathcal{O}(\frac{1}{\sqrt{T}})$ regardless of $D_0=0$ and there is a gap between the rates of AdaGrad and SGD. We believe that closing this gap is important, as the strong growth condition closely resembles the over-parameterized models, which is widely adopted in deep learning. Second, 

We know that the convergence analysis of SGD under the same set of assumptions is quite simple. \textbf{What makes the analysis of AdaGrad so complicated?} We can understand the difficulty from the classical descent lemma
\begin{equation}
\label{eq: descent_main}
 \begin{matrix}\mathbb{E}[f(\bw_{t+1})|\mathcal{F}_t]\le f(\bw_{t}) +\\ \quad
   \end{matrix}\begin{matrix}\underbrace{  \mathbb{E}\left[\left\langle \nabla f(\bw_t), \bw_{t+1}-\bw_t \right\rangle|\mathcal{F}_t\right]}
   \\
    \text{First Order}
   \end{matrix} \begin{matrix}+\\ \quad
   \end{matrix} \begin{matrix}\underbrace{\frac{L}{2}\mathbb{E}\left[ \left\Vert   \bw_{t+1}-\bw_t \right\Vert^2|\mathcal{F}_t\right]}
   \\
    \text{Second Order}
   \end{matrix},
\end{equation}
where $\mathcal{F}_t:=\sigma (g_1,\cdots,g_{t-1})$ denotes the sigma field of the stochastic gradients up to $t-1$.  
%It is due to the correlation between $g_t$ and $\bnu_t$. We illustrate this using AdaGrad-Norm. Specifically, with smoothness uniformly bounded by $L$ (i.e., Assumption \ref{assum: smooth}) the classical descent lemma yields
Then 
\begin{itemize}
    \item for SGD, $\bw_{t+1}-\bw_t=-\eta g_t$ and hence the "First Order" term is $-\eta \Vert \nabla f(\bw_t)\Vert^2 $, which is negative and able to decrease the objective sufficiently,
    \item for AdaGrad(-Norm), $\bw_{t+1}-\bw_t=-\eta \frac{g_t}{\sqrt{\bnu_t }}$. As $\bnu_t$ correlates with $g_t$, the ``First Order" term does not  admit a clear form. 
\end{itemize}
To deal with the correlation in AdaGrad(-Norm), a common practice is to use a surrogate $\tilde{\bnu}_t$ of $\bnu_t$ \citep{ward2020adagrad,defossez2020simple,faw2022power}, which is measurable with respect to $\mathcal{F}_t$, to decompose the ``First Order" term as follows,
\begin{align*}
   & \mathbb{E}\left[\left\langle \nabla f(\bw_t), \bw_{t+1}-\bw_t \right\rangle|\mathcal{F}_t\right]=\mathbb{E}\left[\left\langle \nabla f(\bw_t), -\eta \frac{g_t}{\sqrt{\bnu_t}} \right\rangle|\mathcal{F}_t\right]
    =\mathbb{E}\left[\left\langle \nabla f(\bw_t), -\eta \frac{g_t}{\sqrt{\tilde{\bnu}_t}} \right\rangle|\mathcal{F}_t\right]
    \\
    &+ \mathbb{E}\left[\left\langle \nabla f(\bw_t), \eta g_t \left(\frac{1}{\sqrt{\tilde{\bnu}_t}}-\frac{1}{\sqrt{\bnu_t}}\right) \right\rangle|\mathcal{F}_t\right].
\end{align*}
The first term equals $-\eta \frac{\Vert \nabla f(\bw_t)\Vert^2}{\sqrt{\tilde{\bnu}_t}}$, which is negative and desired. However, the last term is an additional error term, which is very challenging to deal with. Existing results either assume bounded stochastic gradient to work around it \citep{ward2020adagrad}, or resolve it through complicated analysis \citep{faw2022power} (c.f. Section \ref{sec: motivation}).
\\
\textbf{Contributions.} In this paper, we propose a novel auxiliary function $\xi(t)=\frac{\Vert \nabla f(\bw_t) \Vert^2}{\sqrt{\bnu_t}}$ for the convergence analysis of AdaGrad(-Norm), and show the error term can be bounded by $\mathbb{E}^{|\mathcal{F}_t}[\xi(t-1)-\xi(t)]$ (c.f. 
 Lemma \ref{lem: auxiliary}), which can be reduced by telescoping. As explained in Section \ref{sec: motivation}, such a auxiliary function is rooted in the non-increasing nature of the adaptive learning rate $\frac{\eta}{\sqrt{\bnu_{k}}}$.

With the new and simplified proof, we are able to obtain stronger results for AdaGrad-Norm and extend the analysis to other important scenarios.

\begin{itemize}
    \item Under strong growth condition (or the so-called over-parameterized regime), our convergence rate for AdaGrad-Norm is $\mathcal{O}(\frac{1}{T})$, which matches that of SGD and stronger than existing results \citep{faw2022power}. This demonstrates that AdaGrad-Norm converges faster in the over-parameterized regime than in the under-parameterized regime.

    \item We extend the analysis to AdaGrad by utilizing a coordinate version { $\tilde{\xi}(t)=\sum_{l=1}^d \frac{\partial_l f(\bw_t)^2}{\sqrt{\bnu_{t,l}}}$} of $\xi(t)$  and obtain similar convergence.   To the best of our knowledge, this is the first convergence result of AdaGrad without the requirement of bounded gradient norm. We also prove the convergence for randomly-reshuffled AdaGrad, which is the version of AdaGrad used in deep learning practice.

    \item We go beyond the uniform smoothness and consider a realistic non-uniformly smooth condition called $(L_0,L_1)$-smooth condition (Assumption \ref{assum: non-smooth}). We prove that AdaGrad(-Norm) still converges under $(L_0,L_1)$-smooth condition, but requires the learning rate smaller than a threshold, whose necessity is conversely verified with a counterexample. Together, AdaGrad can converge under the non-uniform smoothness but may not be exactly tuning-free.
\end{itemize}

We have observed a concurrent work by \citep{faw2023beyond}, which also establishes the convergence of AdaGrad under $(L_0,L_1)$-smooth condition and affine noise variance assumption. Though results are similar and both appear in COLT 2023, there are some notable differences between their findings and ours. First, \citet{faw2023beyond} require either $D_1<2$ or an additional assumption on the objective function, whereas our result holds for all $D_1$ values without any additional assumption. Second, their result is based on a novel stopping time, while ours relies on a new auxiliary function. Lastly, \citet{faw2023beyond} establish a set of negative results for Clipped SGD and Sign SGD (with momentum) when analyzed under the $(L_0, L_1)$-smooth condition, highlighting the advantage of AdaGrad over these optimizers. All in all, both their work and ours complement each other, providing a more comprehensive understanding of AdaGrad. 
\\
\textbf{Organization of this paper. }The rest of this paper is organized as follows. In Section \ref{sec: preliminary}, we define notations and introduce assumptions; in Section \ref{sec: motivation}, we describe the motivation to use the auxiliary function; in Section \ref{sec: adagrad_norm}, we derive the convergence result of AdaGrad-Norm under $L$-smooth condition; in Section \ref{sec: adagrad}, we extend the result to AdaGrad; in Section \ref{sec: non-uniform}, we analyze the convergence of AdaGrad(-Norm) under $(L_0,L_1)$-smooth condition; Section \ref{sec: related} presents the related works.

\section{Preliminary}
\label{sec: preliminary}
\textbf{Notations.} The following notations are used throughout this paper.
\begin{itemize} 
    \item (Vector operators) $\odot$ stands for the Hadamard product between vectors, and  $g^{\odot2}\triangleq g\odot g$.  $\langle \bw,\boldsymbol{v}\rangle $ stands for the $L^2$ inner product between $\bw$ and $\bv$, and $\Vert \bw\Vert \triangleq \sqrt{\langle \bw,\bw \rangle}$ .

    \item (Stochastic operators)  $\mathcal{F}_t= \sigma (g_{t-1},\cdots, g_1)$ stands for the sigma field of historical gradients up to time $t-1$ and thus $\{\bw_t\}_{t=1}^{\infty}$ is an adapted random process with respect to $\{\mathcal{F}_t\}_{t=1}^{\infty}$. For brevity, we abbreviate the expectation conditional on $\mathcal{F}_t$ as $\mathbb{E}^{\vert \mathcal{F}_t} [*]\triangleq \mathbb{E}[*|\mathcal{F}_t]$.
\end{itemize}

\textbf{Assumptions.} Throughout this paper, we assume that $f$ is lower bounded. We also need the following assumptions:

\begin{assumption}[$L$-smooth condition]
\label{assum: smooth}
We assume that $f$ is differentiable and its gradient satisfies that $\forall \bw_1,\bw_2 \in \mathbb{R}^d$, we have 
   $\Vert  \nabla f(\bw_1) - \nabla f(\bw_2) \Vert \le L \Vert \bw_1-\bw_2 \Vert$.
\end{assumption}

\begin{assumption}[Affine noise variance]
\label{assum: affine}
We assume that there exist positive constants $D_0$ and $D_1$ such that $\forall t \ge 1$,
    $\mathbb{E}^{| \mathcal{F}_t}[\Vert g_t\Vert^2] \le D_0+D_1 \Vert \nabla f(\bw_t) \Vert^2$.

\end{assumption}
To the best of our knowledge, the above two assumptions are the weakest requirements for the convergence of AdaGrad(-Norm) among the existing literature.

\section{Motivation of the auxiliary function}
\label{sec: motivation}
% \huishuai{I suggest adding some context of proving the convergence with auxiliary function, and list some choices of auxiliary functions, which may show the speciality of our choice. Something like Section 1.1 in \hyperlink{https://arxiv.org/pdf/1712.04581.pdf}{this summary paper}.}
{ 
 As mentioned in Introduction, the main obstacle in the analysis of AdaGrad(-Norm) is to bound the error term $\mathbb{E}^{|\mathcal{F}_t}\langle \nabla f(\bw_t), \eta g_t (\frac{1}{\sqrt{\tilde{\bnu}_t}}-\frac{1}{\sqrt{\bnu_t}}) \rangle$. Most of the existing works assume that $\Vert g_t \Vert $ is uniformly bounded, and choose $\tilde{\bnu}_{t}=\bnu_{t-1}$. In this case, the error term can be shown to be as small as the "Second Order" term in Eq. (\ref{eq: descent_main}) and can be further bounded. If the bounded gradient assumption is removed, \citet{faw2022power} shows that \emph{most of the iterations are "good"}, in the sense that the error term is smaller than $\eta \frac{\Vert \nabla f(\bw_t)\Vert^2}{\sqrt{\tilde{\bnu}_t}}$, which hence won't affect the negativity of  the "First Order" term in Eq. (\ref{eq: descent_main}) after decomposition.  However, it is complicated to deal  with the "bad" iterations, which occupies the main space of the proof in \citep{faw2022power}.
 }
 %Therefore, \emph{most of the iterations}, the "First Order" term is negative and desired.
 %=\frac{\Vert \nabla f(\bw_t) \Vert^2 }{\sqrt{\bnu_{t}}}

Instead, to deal with the error term, we propose a simple auxiliary function $\xi(t)$  that can be canceled out during telescoping. The choice of $\xi(t)$ is motivated as follows. By choosing $\tilde{\bnu}_t=\bnu_{t-1}$, we find that the error term can be rewritten as
\begin{align}
\mathbb{E}^{|\mathcal{F}_t}\left\langle \nabla f(\bw_t), \eta g_t \left(\frac{1}{\sqrt{\bnu_{t-1}}}-\frac{1}{\sqrt{\bnu_t}}\right) \right\rangle\le & \eta\mathbb{E}^{|\mathcal{F}_t}\left[\Vert \nabla f(\bw_t) \Vert \Vert g_t \Vert  \left(\frac{1}{\sqrt{\bnu_{t-1}}}-\frac{1}{\sqrt{\bnu_t}}\right)\right] \nonumber
\\
=& \eta\mathbb{E}^{|\mathcal{F}_t}\left[  \left(\frac{\Vert \nabla f(\bw_t) \Vert \Vert g_t \Vert}{\sqrt{\bnu_{t-1}}}-\frac{\Vert \nabla f(\bw_t) \Vert \Vert g_t \Vert}{\sqrt{\bnu_t}}\right)\right], \label{eq:error-term}
\end{align}
where the inequality is due to the Cauchy-Schwarz inequality and $\bnu_t$ is non-decreasing. Note that if we have both $\Vert \nabla f(\bw_t)\Vert \approx \Vert \nabla f(\bw_{t-1})\Vert$ and $\Vert g_t \Vert \approx \Vert g_{t-1} \Vert $, the term \eqref{eq:error-term} approximately equals to $\eta\mathbb{E}^{|\mathcal{F}_t}\left[  \left(\frac{\Vert \nabla f(\bw_{t-1}) \Vert \Vert g_{t-1} \Vert}{\sqrt{\bnu_{t-1}}}-\frac{\Vert \nabla f(\bw_t) \Vert \Vert g_t \Vert}{\sqrt{\bnu_t}}\right)\right]$. { In this case, we can use $\hat{\xi}(t)=\frac{\Vert \nabla f(\bw_t) \Vert \Vert g_t \Vert}{\sqrt{\bnu_t}}$ as an auxiliary function, and the sum of the expected error term satisfies
\begin{equation*}
    \sum_{t=1}^T\mathbb{E}\left\langle \nabla f(\bw_t), \eta g_t \left(\frac{1}{\sqrt{\bnu_{t-1}}}-\frac{1}{\sqrt{\bnu_t}}\right) \right\rangle\lesssim \sum_{t=1}^T \mathbb{E} \left[\hat{\xi}(t-1)-\hat{\xi}(t)\right]=\hat{\xi}(0)-\mathbb{E}[\hat{\xi}(T)].
\end{equation*}
The RHS of the above inequality is bounded regardless of $T$.
This is the motivation to use the auxiliary function.} However, we do not have $\Vert g_t \Vert \approx \Vert g_{t-1} \Vert $ but only have  $\Vert \nabla f(\bw_t)\Vert \approx \Vert \nabla f(\bw_{t-1})\Vert$ (due to bounded smoothness, i.e., Assumption \ref{assum: smooth}). To resolve this challenge, we convert $\Vert g_t \Vert $ to $\Vert \nabla f(\bw_t) \Vert$ by Assumption \ref{assum: affine} in the above inequality, and use $\xi(t)\triangleq \frac{\Vert \nabla f(\bw_t) \Vert^2 }{\sqrt{\bnu_t}}$ instead of $\frac{\Vert \nabla f(\bw_t) \Vert \Vert g_t \Vert}{\sqrt{\bnu_t}}$ as the auxiliary function. A formal statement of the above methodology can be seen in  Lemma \ref{lem: auxiliary}.

\begin{remark}
    Note that the above methodology is mainly based on that the adaptive learning rate is non-increasing. Therefore, we believe that similar approach can be applied to the analysis of other adaptive optimizers with non-increasing adaptive learning rates, such as AMSGrad.
\end{remark}

\section{A refined convergence analysis of AdaGrad-Norm}
\label{sec: adagrad_norm}
In this section, we present our refined analysis of AdaGrad-Norm based on the auxiliary function $\xi(t)$. The refined convergence rate is given by the following theorem.
\begin{theorem}
\label{thm: adagrad_norm}
    Let Assumptions \ref{assum: smooth} and \ref{assum: affine} hold. Then, for AdaGrad-Norm with any learning rate $ \eta >0$, we have that with probability at least $1-\delta$,
    \begin{equation*}
       \min_{t\in [T]} \Vert \nabla f(\bw_t) \Vert^2 \le \frac{2\sqrt{2D_0}(2C_2\ln(2\sqrt{2D_0T}+C_3)+C_1)}{\sqrt{T}\delta^2}+ \frac{C_3(C_1+2C_2\ln(2\sqrt{2D_0T}+C_3))}{T\delta^2},
    \end{equation*}
    where $C_1$, $C_2$, and $C_3$ are constants defined as $C_1:=4 ( f(\bw_1)-f^*+\frac{\eta D_1}{2} \frac{\Vert \nabla f(\bw_0) \Vert^2}{\sqrt{\bnu_0}}
    +(2  \eta (L \eta D_1)^2 +\eta D_1 (L \eta)^2 +\frac{\eta}{2} D_0 )\frac{1}{\sqrt{\bnu_0}}-\frac{L}{2}\eta^2 \ln \bnu_0)/\eta $, $C_2:= 2L\eta$, and $C_3:= 4 D_1C_1 +48C_2 D_1 \ln  (4C_2 D_1+e)+2\sqrt{\bnu_0}$.
\end{theorem}
\begin{remark}
 \citet{faw2022power} also prove that AdaGrad-Norm converges under Assumptions \ref{assum: smooth} and \ref{assum: affine}. Their rate is $\frac{1}{T}\sum_{t=1}^T \Vert \nabla f(\bw_t) \Vert^2=\mathcal{O}(\frac{\log^{\frac{9}{4}} T}{\sqrt{T}})$. Compared to their result, our result has a tighter dependence over $T$. Moreover, when restricted to the strong growth condition, i.e., $D_0=0$, our result gives a rate $\mathcal{O}(\frac{1}{T})$,  much faster than that in \citep{faw2022power} and matching that of SGD. Such an improvement counts as  strong growth condition characterizes the landscapes of over-parameterized models \citep{vaswani2019fast}. Theorem \ref{thm: adagrad_norm} also shows that AdaGrad-Norm enjoys the tuning-free ability under $L$-smooth condition, i.e., it converges without tuning the learning rate. 
\end{remark}
%is able to interpolate the over-parameterized regime ($D_0=0$) and the under-parameterized regime ($D_0\ne 0$)
\textbf{Proof of Theorem \ref{thm: adagrad_norm}}
The proof starts with the so-called expected descent lemma:
\begin{align}
\nonumber
    \mathbb{E}^{|\mathcal{F}_t}f(\bw_{t+1})\le& f(\bw_{t})+ \mathbb{E}^{|\mathcal{F}_t}\left[\langle \nabla f(\bw_t), \bw_{t+1}-\bw_t \rangle + \frac{L}{2} \Vert \bw_{t+1}-\bw_{t} \Vert^2\right]
    \\
\label{eq: descent lemma}
    = & f(\bw_{t}) +\begin{matrix}\underbrace{  \mathbb{E}^{|\mathcal{F}_t}\left\langle \nabla f(\bw_t), -\eta\frac{g_t}{\sqrt{\bnu_t} } \right\rangle}
   \\
    \text{First Order}
   \end{matrix} + \begin{matrix}\underbrace{\frac{L}{2}\eta^2\mathbb{E}^{|\mathcal{F}_t} \left\Vert  \frac{g_t}{\sqrt{\bnu_t} } \right\Vert^2}
   \\
    \text{Second Order}
   \end{matrix} . 
\end{align}

As discussed in Section \ref{sec: intro}, the ``First Order'' term does not have a simple form due to the correlation between $g_t$ and $\bnu_t$. We follow the standard approach in existing literature to approximate $\bnu_t$ with the surrogate  $\bnu_{t-1}$, which is measurable with respect to $\mathcal{F}_t$. The first-order term can then be decomposed into
\small
\begin{align}
\nonumber
    \mathbb{E}^{|\mathcal{F}_t}\left[\left\langle \nabla f(\bw_t), -\eta \frac{g_t}{\sqrt{\bnu_t} } \right\rangle\right]
    = & \mathbb{E}^{|\mathcal{F}_t}\left[\left\langle \nabla f(\bw_t), -\eta \frac{g_t}{\sqrt{\bnu_{t-1}} } \right\rangle\right]+\mathbb{E}^{|\mathcal{F}_t}\left[\left\langle \nabla f(\bw_t), \eta (\frac{1}{\sqrt{\bnu_{t-1}} }-\frac{1}{\sqrt{\bnu_{t}} })g_t \right\rangle\right]
    \\
\label{eq: decomposition_approximation}
    =& -\eta\frac{\Vert \nabla f(\bw_t)\Vert^2 }{\sqrt{\bnu_{t-1}} } +\mathbb{E}^{|\mathcal{F}_t}\left[\left\langle \nabla f(\bw_t), \eta \left(\frac{1}{\sqrt{\bnu_{t-1}} }-\frac{1}{\sqrt{\bnu_{t}} }\right)g_t \right\rangle\right].
\end{align}
\normalsize
The last term is an error term, coming from the gap between $\bnu_{t-1}$ and $\bnu_t$. Plugging  Eq. (\ref{eq: decomposition_approximation}) back to Eq. (\ref{eq: descent lemma}) and we obtain
\small
\begin{align}
\nonumber
    \mathbb{E}^{|\mathcal{F}_t}f(\bw_{t+1})\le  f(\bw_{t}) + \begin{matrix}\underbrace{-\eta\frac{\Vert \nabla f(\bw_t)\Vert^2 }{\sqrt{\bnu_{t-1}} }}
   \\
    \text{First Order Main}
    \end{matrix} + \begin{matrix}\underbrace{\mathbb{E}^{|\mathcal{F}_t}\left[\left\langle \nabla f(\bw_t), \eta \left(\frac{1}{\sqrt{\bnu_{t-1}} }-\frac{1}{\sqrt{\bnu_{t}} }\right)g_t \right\rangle\right] }
   \\
    \text{Error}
    \end{matrix} + \begin{matrix}\underbrace{\frac{L}{2}\eta^2\mathbb{E}^{|\mathcal{F}_t} \left\Vert  \frac{g_t}{\sqrt{\bnu_t} } \right\Vert^2}
   \\
    \text{Second Order}
   \end{matrix} .
\end{align}
\normalsize
% Similar to existing analyses of adaptive optimizers \citep{ward2020adagrad,defossez2020simple,faw2022power}, our proof can be divided into two stages: In the first stage, we aim to bound $\mathbb{E}\sum_{t=1}^T \frac{\Vert \nabla f(\bw_t)\Vert^2}{\sqrt{\bnu_{t-1}}}$.  Then, in the second stage, we convert the bound of $\mathbb{E}\sum_{t=1}^T \frac{\Vert \nabla f(\bw_t)\Vert^2}{\sqrt{\bnu_{t-1}}}$ into the bound of $\sum_{t=1}^T \Vert \nabla f(\bw_t) \Vert^2$.
The rest of the proof can be divided into two stages: in Stage I, We proceed by bounding the ``Error" term through the auxiliary function $\xi(t)\triangleq \frac{\Vert \nabla f(\bw_t) \Vert^2}{\sqrt{\bnu_t}}$, and bound $\sum_{t=1}^T \frac{\Vert \nabla f(\bw_t)\Vert^2}{\sqrt{\bnu_{t-1}}}$. In Stage II, we convert the  bound of $\sum_{t=1}^T \frac{\Vert \nabla f(\bw_t)\Vert^2}{\sqrt{\bnu_{t-1}}}$ into the bound of $\sum_{t=1}^T \Vert \nabla f(\bw_t) \Vert^2$.
\\
\textbf{Stage I: Bounding the ``Error" term.} The following lemma summarizes the intuition in Section \ref{sec: motivation}.

\begin{lemma}
\label{lem: auxiliary}
Define an auxiliary function $\xi(t)\triangleq \frac{\Vert \nabla f(\bw_t)\Vert^2 }{\sqrt{\bnu_t}}$, $t\ge 1$. Then, the ``Error" term can be bounded as
\begin{align*}
    &\mathbb{E}^{|\mathcal{F}_t}\left[\left\langle \nabla f(\bw_t), \eta \left(\frac{1}{\sqrt{\bnu_{t-1}} }-\frac{1}{\sqrt{\bnu_{t}} }\right)g_t \right\rangle\right]    \le \frac{3}{4}\eta\frac{\Vert \nabla f(\bw_t)\Vert^2}{\sqrt{\bnu_{t-1}} }+\frac{1}{2}\frac{\eta}{\sqrt{\bnu_{t-1}} }  D_0 \mathbb{E}^{|\mathcal{F}_t}\left[  \frac{\Vert g_t\Vert^2}{(\sqrt{\bnu_{t}}+\sqrt{\bnu_{t-1}})^2}\right]
    \\
    &+\frac{\eta}{2} D_1  \mathbb{E}^{|\mathcal{F}_t}\left[\xi(t-1)-\xi(t)\right]
    +\left(\eta (L \eta D_1)^2 +\frac{\eta}{2} D_1 (L \eta)^2\right)  \frac{\Vert g_{t-1} \Vert^2}{\sqrt{\bnu_{t-1}}^3} .
\end{align*}

\end{lemma}

\begin{proof}
By a simple calculation, we have
\small
\begin{align}
\nonumber
   & \left\vert\mathbb{E}^{|\mathcal{F}_t}\left[\left\langle \nabla f(\bw_t), \eta \left(\frac{1}{\sqrt{\bnu_{t-1}} }-\frac{1}{\sqrt{\bnu_{t}} }\right)g_t \right\rangle\right]\right\vert
=  \left\vert\mathbb{E}^{|\mathcal{F}_t}\left[\left\langle \nabla f(\bw_t), \eta \frac{\Vert g_t\Vert^2}{(\sqrt{\bnu_{t-1}} )(\sqrt{\bnu_{t}} )(\sqrt{\bnu_{t}}+\sqrt{\bnu_{t-1}})}g_t \right\rangle\right] \right\vert
\\
\le &\eta\mathbb{E}^{|\mathcal{F}_t}\left[\Vert \nabla f(\bw_t)\Vert  \frac{\Vert g_t\Vert^3}{(\sqrt{\bnu_{t-1}} )(\sqrt{\bnu_{t}} )(\sqrt{\bnu_{t}}+\sqrt{\bnu_{t-1}})}\right] 
\nonumber
\le  \eta\frac{\Vert \nabla f(\bw_t)\Vert}{\sqrt{\bnu_{t-1}} }\mathbb{E}^{|\mathcal{F}_t}\left[  \frac{\Vert g_t\Vert^2}{\sqrt{\bnu_{t}}+\sqrt{\bnu_{t-1}}}\right]
,
\end{align}
\normalsize
where the first inequality is due to Cauchy-Schwarz inequality, and the second inequality is because $\bnu_t\ge \Vert g_t\Vert^2$. By the mean-value inequality ($2ab\le a^2+b^2$), 
\begin{equation}
\label{eq: before_cauchy}
    \eta\frac{\Vert \nabla f(\bw_t)\Vert}{\sqrt{\bnu_{t-1}} }\mathbb{E}^{|\mathcal{F}_t}\left[  \frac{\Vert g_t\Vert^2}{\sqrt{\bnu_{t}}+\sqrt{\bnu_{t-1}}}\right] \le \frac{1}{2}\eta\frac{\Vert \nabla f(\bw_t)\Vert^2}{\sqrt{\bnu_{t-1}} }+\frac{1}{2}\frac{\eta}{\sqrt{\bnu_{t-1}} }\left(\mathbb{E}^{|\mathcal{F}_t}\left[  \frac{\Vert g_t\Vert^2}{\sqrt{\bnu_{t}}+\sqrt{\bnu_{t-1}}}\right]\right)^2.
\end{equation}
We focus on the last quantity $(\mathbb{E}^{|\mathcal{F}_t}[  \frac{\Vert g_t\Vert^2}{\sqrt{\bnu_{t}}+\sqrt{\bnu_{t-1}}}])^2$. By further applying Hölder's inequality,
\begin{align*}
   & \left(\mathbb{E}^{|\mathcal{F}_t}\left[  \frac{\Vert g_t\Vert^2}{\sqrt{\bnu_{t}}+\sqrt{\bnu_{t-1}}}\right]\right)^2\le \mathbb{E}^{|\mathcal{F}_t}  \Vert g_t \Vert^2 \cdot\mathbb{E}^{|\mathcal{F}_t}\left[  \frac{\Vert g_t\Vert^2}{(\sqrt{\bnu_{t}}+\sqrt{\bnu_{t-1}})^2}\right]
    \\
  & ~~~~~~~~~~~~~~~~~~~~~~~~~~~~~~~~~~~~~~~~~~~~~~~~~~~~~~~~~~~~~~~~~~~~~~\le ( D_0+ D_1 \Vert \nabla f(\bw_{t}) \Vert^2)  \mathbb{E}^{|\mathcal{F}_t}\left[  \frac{\Vert g_t\Vert^2}{(\sqrt{\bnu_{t}}+\sqrt{\bnu_{t-1}})^2}\right] ,
\end{align*}
\normalsize
where in the last inequality we use Assumption \ref{assum: smooth}.
Plugging the above inequality back to Eq. (\ref{eq: before_cauchy}),  the ``Error" term can be bounded as 

\begin{align}
\nonumber
   & \left\vert\mathbb{E}^{|\mathcal{F}_t} \left[\langle \nabla f(\bw_t), \eta (\frac{1}{\sqrt{\bnu_{t-1}} }-\frac{1}{\sqrt{\bnu_{t}} })g_t \rangle\right]\right \vert \le \frac{1}{2}\eta\frac{\Vert \nabla f(\bw_t)\Vert^2}{\sqrt{\bnu_{t-1}} }+\frac{1}{2}\frac{\eta}{\sqrt{\bnu_{t-1}} }  D_0 \mathbb{E}^{|\mathcal{F}_t}\left[  \frac{\Vert g_t\Vert^2}{(\sqrt{\bnu_{t}}+\sqrt{\bnu_{t-1}})^2}\right]
    \\
\label{eq: mid_lemma}
     & +\frac{1}{2}\frac{\eta}{\sqrt{\bnu_{t-1}} }  D_1\Vert \nabla f(\bw_t)\Vert^2\ \mathbb{E}^{|\mathcal{F}_t}\left[  \frac{\Vert g_t\Vert^2}{(\sqrt{\bnu_{t}}+\sqrt{\bnu_{t-1}})^2}\right].
\end{align}
In the RHS of the above equality, the first term is a negative half of the ``First  Order Main" term, and the second term is $\frac{1}{\eta L \sqrt{\bnu_{t-1}}}$ times of the ``Second order" term, and thus is at the same order of the ``Second order" term due to $\frac{1}{\sqrt{\bnu_{t-1}}}$ is upper bounded. We focus on the last term, and utilize the observation that
\begin{align*}
\frac{\Vert g_t\Vert^2}{\sqrt{\bnu_{t-1}}(\sqrt{\bnu_{t}}+\sqrt{\bnu_{t-1}})^2}
    \le \frac{\Vert g_t\Vert^2}{\sqrt{\bnu_{t-1}} \sqrt{\bnu_{t}}(\sqrt{\bnu_{t}}+\sqrt{\bnu_{t-1}})}=\frac{1}{\sqrt{\bnu_{t-1}}}-\frac{1}{\sqrt{\bnu_{t}}}. 
\end{align*}
Thus, the last term can be bounded as
\begin{align*}
    &\frac{1}{2}\frac{\eta D_1 \Vert \nabla f(\bw_t) \Vert^2}{\sqrt{\bnu_{t-1}} } \mathbb{E}^{|\mathcal{F}_t}\left[  \frac{\Vert g_t\Vert^2}{(\sqrt{\bnu_{t}}+\sqrt{\bnu_{t-1}})^2}\right]
    \le \frac{1}{2}\eta D_1 \Vert \nabla f(\bw_t) \Vert^2 \mathbb{E}^{|\mathcal{F}_t}\left(\frac{1}{\sqrt{\bnu_{t-1}}}-\frac{1}{\sqrt{\bnu_{t}}}\right),
\end{align*}
which can be further decomposed into
\begin{align*}
    &\frac{1}{2}\eta D_1 \Vert \nabla f(\bw_t) \Vert^2 \mathbb{E}^{|\mathcal{F}_t}\left(\frac{1}{\sqrt{\bnu_{t-1}}}-\frac{1}{\sqrt{\bnu_{t}}}\right)
    \\
    =&  \frac{\eta}{2} D_1  \mathbb{E}^{|\mathcal{F}_t}\left(\frac{\Vert \nabla f(\bw_{t-1}) \Vert^2}{\sqrt{\bnu_{t-1}}}-\frac{\Vert \nabla f(\bw_t) \Vert^2}{\sqrt{\bnu_{t}}}\right)
    +\frac{\eta}{2} D_1  \frac{\Vert \nabla f(\bw_{t}) \Vert^2-\Vert \nabla f(\bw_{t-1}) \Vert^2}{\sqrt{\bnu_{t-1}}}.
\end{align*}
By  Assumption \ref{assum: smooth}, $\Vert \nabla f(\bw_{t}) \Vert-\Vert \nabla f(\bw_{t-1}) \Vert \le \Vert  \nabla f(\bw_{t})- \nabla f(\bw_{t-1}) \Vert \le L \Vert \bw_t-\bw_{t-1} \Vert  $. Therefore, 
\begin{align*}
    &\frac{1}{2}\eta D_1 \Vert \nabla f(\bw_t) \Vert^2 \mathbb{E}^{|\mathcal{F}_t}\left(\frac{1}{\sqrt{\bnu_{t-1}}}-\frac{1}{\sqrt{\bnu_{t}}}\right)
    \\
    =&  \frac{\eta}{2} D_1  \mathbb{E}^{|\mathcal{F}_t}\left(\frac{\Vert \nabla f(\bw_{t-1}) \Vert^2}{\sqrt{\bnu_{t-1}}}-\frac{\Vert \nabla f(\bw_t) \Vert^2}{\sqrt{\bnu_{t}}}\right)
    +\frac{\eta}{2} D_1  \frac{2L\Vert \bw_{t}-\bw_{t-1}\Vert \Vert \nabla f(\bw_t)\Vert +L^2 \Vert \bw_{t}-\bw_{t-1}\Vert^2}{\sqrt{\bnu_{t-1}}}
    \\
    = & \frac{\eta}{2} D_1  \mathbb{E}^{|\mathcal{F}_t}(\xi(t-1)-\xi(t))
    +\frac{\eta}{2} D_1  \frac{2L \eta \frac{\Vert g_{t-1} \Vert \Vert \nabla f(\bw_t)\Vert}{\sqrt{\bnu_{t-1}}}+(L \eta \frac{\Vert g_{t-1} \Vert}{\sqrt{\bnu_{t-1}}})^2}{\sqrt{\bnu_{t-1}}},
\end{align*}
where the last inequality we use $\xi(t)\triangleq \frac{\Vert \nabla f(\bw_t)\Vert^2}{\sqrt{\bnu_t}}$ and $\bw_t-\bw_{t-1}=\eta\frac{g_{t-1}}{\sqrt{\bnu_{t-1}}}$. Applying again the Cauchy-Schwarz inequality, we obtain
\small
\begin{align}
\nonumber
    &\frac{\eta}{2} D_1 \mathbb{E}^{|\mathcal{F}_t}\Vert \nabla f(\bw_t) \Vert^2 \left(\frac{1}{\sqrt{\bnu_{t-1}}}-\frac{1}{\sqrt{\bnu_{t}}}\right)
    \\
    \label{eq: challenge}
    \le & \frac{\eta}{2} D_1  \mathbb{E}^{|\mathcal{F}_t}\left(\xi(t-1)-\xi(t)\right)
    +\frac{1}{4} \eta \frac{\Vert \nabla f(\bw_t) \Vert^2}{\sqrt{\bnu_{t-1}}} +\eta (L \eta D_1)^2   \frac{\Vert g_{t-1} \Vert^2}{\sqrt{\bnu_{t-1}}^3}+\frac{\eta}{2} D_1 (L \eta)^2 \frac{\Vert g_{t-1} \Vert^2}{\sqrt{\bnu_{t-1}}^3}.
\end{align}
\normalsize
Applying the above inequality back into Eq. (\ref{eq: mid_lemma}), the ``Error" term can be bounded as
\begin{align*}
    &\mathbb{E}^{|\mathcal{F}_t}[\langle \nabla f(\bw_t), \eta (\frac{1}{\sqrt{\bnu_{t-1}} }-\frac{1}{\sqrt{\bnu_{t}} })g_t \rangle]    \le \frac{1}{2}\eta\frac{\Vert \nabla f(\bw_t)\Vert^2}{\sqrt{\bnu_{t-1}} }+\frac{1}{2}\frac{\eta}{\sqrt{\bnu_{t-1}} }  D_0 \mathbb{E}^{|\mathcal{F}_t}\left[  \frac{\Vert g_t\Vert^2}{(\sqrt{\bnu_{t}}+\sqrt{\bnu_{t-1}})^2}\right]
    \\
    &+\frac{\eta}{2} D_1  \mathbb{E}^{|\mathcal{F}_t}\left(\xi(t-1)-\xi(t)\right)
    +\frac{1}{4} \eta \frac{\Vert \nabla f(\bw_t) \Vert^2}{\sqrt{\bnu_{t-1}}} +\eta (L \eta D_1)^2   \frac{\Vert g_{t-1} \Vert^2}{\sqrt{\bnu_{t-1}}^3}+\frac{\eta}{2} D_1 (L \eta)^2 \frac{\Vert g_{t-1} \Vert^2}{\sqrt{\bnu_{t-1}}^3}.
\end{align*}
Rearranging the RHS of the above inequality leads to the claim.
\end{proof}

Applying Lemma \ref{lem: auxiliary} back to the descent lemma, we then have
\begin{small}
\begin{align}
\nonumber
    &\mathbb{E}^{|\mathcal{F}_t}[f(\bw_{t+1})]
    \le f(\bw_t) -\eta \frac{\Vert\nabla f(\bw_t) \Vert^2 }{\sqrt{\bnu_{t-1}}}+\frac{3}{4}\eta\frac{\Vert \nabla f(\bw_t)\Vert^2}{\sqrt{\bnu_{t-1}} }+\frac{1}{2}\frac{\eta}{\sqrt{\bnu_{t-1}} }  D_0 \mathbb{E}^{|\mathcal{F}_t}\left[  \frac{\Vert g_t\Vert^2}{(\sqrt{\bnu_{t}}+\sqrt{\bnu_{t-1}})^2}\right]
    \\
\nonumber
    &+\frac{\eta}{2} D_1  \mathbb{E}^{|\mathcal{F}_t}\left(\xi(t-1)-\xi(t)\right)
    +\left(\eta (L \eta D_1)^2 +\frac{\eta}{2} D_1 (L \eta)^2\right)  \frac{\Vert g_{t-1} \Vert^2}{\sqrt{\bnu_{t-1}}^3} +\frac{L}{2}\eta^2\mathbb{E}^{|\mathcal{F}_t} \left\Vert  \frac{g_t}{\sqrt{\bnu_t} } \right\Vert^2
    \\
\nonumber
    =&  f(\bw_t) -\frac{1}{4}\eta \frac{\Vert\nabla f(\bw_t) \Vert^2 }{\sqrt{\bnu_{t-1}}}+\frac{1}{2}\frac{\eta}{\sqrt{\bnu_{t-1}} }  D_0 \mathbb{E}^{|\mathcal{F}_t}\left[  \frac{\Vert g_t\Vert^2}{(\sqrt{\bnu_{t}}+\sqrt{\bnu_{t-1}})^2}\right]+\frac{\eta}{2} D_1  \mathbb{E}^{|\mathcal{F}_t}\left(\xi(t-1)-\xi(t)\right)
    \\
    &
    \label{eq: potantial_mid_2}
    +\left(\eta (L \eta D_1)^2 +\frac{\eta}{2} D_1 (L \eta)^2\right)  \frac{\Vert g_{t-1} \Vert^2}{\sqrt{\bnu_{t-1}}^3} +\frac{L}{2}\eta^2\mathbb{E}^{|\mathcal{F}_t} \left\Vert  \frac{g_t}{\sqrt{\bnu_t} } \right\Vert^2.
\end{align}
\end{small}
Taking expectation with respect to $\mathcal{F}_t$ to the above inequality then leads to 
\begin{align*}
    \mathbb{E}[f(\bw_{t+1})]
    \le&  \mathbb{E}[f(\bw_t)]-\frac{1}{4} \eta \mathbb{E}\left[\frac{\Vert \nabla f(\bw_t) \Vert^2}{\sqrt{\bnu_{t-1}}}\right]+ \frac{\eta}{2} D_1  \mathbb{E}\left(\xi(t-1)-\xi(t)\right)
    \\
    +& \frac{\eta D_0}{2 } \mathbb{E}\left[  \frac{\Vert g_t\Vert^2}{\sqrt{\bnu_{t-1}}(\sqrt{\bnu_{t}}+\sqrt{\bnu_{t-1}})^2}\right]+\frac{L}{2}\eta^2 \mathbb{E}\left\Vert  \frac{g_t}{\sqrt{\bnu_t} } \right\Vert^2+(\eta (L \eta D_1)^2 +\frac{\eta}{2} D_1 (L \eta)^2 ) \mathbb{E}\frac{\Vert g_{t-1} \Vert^2}{\sqrt{\bnu_{t-1}}^3}.
\end{align*}
The sum over $t$ from $1$ to $T$ of the last three terms above can be bounded by 
\begin{align*}
    &\frac{\eta D_0}{2 } \sum_{t=1}^T\mathbb{E}\left[  \frac{\Vert g_t\Vert^2}{\sqrt{\bnu_{t-1}}(\sqrt{\bnu_{t}}+\sqrt{\bnu_{t-1}})^2}\right]+\sum_{t=1}^T\frac{L}{2}\eta^2 \mathbb{E}\left\Vert  \frac{g_t}{\sqrt{\bnu_t} } \right\Vert^2+(\eta (L \eta D_1)^2 +\sum_{t=1}^T\frac{\eta}{2} D_1 (L \eta)^2 ) \mathbb{E}\frac{\Vert g_{t-1} \Vert^2}{\sqrt{\bnu_{t-1}}^3}
    \\
    \le & \left(2  \eta (L \eta D_1)^2 +\eta D_1 (L \eta)^2 )+\frac{\eta}{2} D_0 \right)\frac{1}{\sqrt{\bnu_0}}+\frac{L}{2}\eta^2 (\mathbb{E}\ln \bnu_T-\ln \bnu_0), 
\end{align*}
where the inequality is due to that if $\{a_i\}_{i=0}^{\infty}$ is a series of non-negative real numbers with $a_0>0$, then  $\sum_{t=1}^T \frac{a_t}{\sqrt{(\sum_{s=0}^t a_s)^3}} \le 2\frac{1}{\sqrt a_0}$,
        $\sum_{t=1}^T \frac{a_t}{\sum_{s=0}^t a_s} \le  \ln \sum_{t=0}^T a_t -\ln a_0$, and
         $\sum_{t=1}^T \frac{a_t}{\sqrt{\sum_{s=0}^t a_s}(\sqrt{\sum_{s=0}^{t-1} a_s}+\sqrt{\sum_{s=0}^t a_s})^2}$ $ \le \frac{1}{\sqrt a_0}$ .
Therefore, summing Eq. (\ref{eq: potantial_mid_2}) over $t$ from $1$ to $T$ leads to
\small
\begin{align}
\nonumber
    &\frac{1}{4} \eta \sum_{t=1}^T \mathbb{E}\frac{\Vert \nabla f(\bw_t) \Vert^2}{\sqrt{\bnu_{t-1}}}\le f(\bw_1)-\mathbb{E}[f(\bw_T)]+ \frac{\eta D_1}{2} \mathbb{E}[\xi(0)-\xi(t)]
    \\
     &~~~~~~~~~~~~~~~~~~~~~~~~~~+\left(2  \eta (L \eta D_1)^2 +\eta D_1 (L \eta)^2 )+\frac{\eta}{2} D_0 \right)\frac{1}{\sqrt{\bnu_0}}+\frac{L}{2}\eta^2 (\mathbb{E}\ln \bnu_T-\ln \bnu_0)
    \label{eq: adaptive gradient bound} .
\end{align}
\normalsize
Applying the definition of $C_1$ and $C_2$, we have $\sum_{t=1}^T \mathbb{E}\frac{\Vert \nabla f(\bw_t) \Vert^2}{\sqrt{\bnu_{t-1}}}\le C_1+C_2 \mathbb{E}\ln \bnu_T$. In Stage II, we translate such an inequality to the bound of $\sum_{t=1}^T \Vert \nabla f(\bw_t)\Vert^2$.
\\
\textbf{Stage II: Bound $\sum_{t=1}^T \Vert \nabla f(\bw_t)\Vert^2$.} We  bound $\mathbb{E}[\sqrt{\bnu_T}]$ by divide-and-conquer. We first consider the iterations satisfying $\nabla f(\bw_t)>\frac{D_0}{D_1}$:
\small
\begin{align}
\nonumber
       C_1 +C_2  \mathbb{E}\ln \bnu_T \ge&   \sum_{t=1}^T \mathbb{E}\left[\frac{\Vert \nabla f(\bw_t) \Vert^2}{\sqrt{\bnu_{t-1}}}\right] \ge     \sum_{t=1}^T \mathbb{E}\left[\frac{\Vert \nabla f(\bw_t) \Vert^2}{\sqrt{\bnu_{t-1}}}\mathds{1}_{\Vert \nabla f(\bw_t)\Vert^2 > \frac{D_0}{D_1}}\right]
        \\
        \label{eq: case_large}
        \ge & \frac{1}{2D_1 } \sum_{t=1}^T \mathbb{E}\left[\frac{\Vert g_t \Vert^2}{\sqrt{\bnu_{t-1}}}\mathds{1}_{\Vert \nabla f(\bw_t)\Vert^2 > \frac{D_0}{D_1}}\right]\ge \frac{1}{2D_1 }   \mathbb{E}\left[\frac{\sum_{t=1}^T\Vert g_t \Vert^2 \mathds{1}_{\Vert \nabla f(\bw_t)\Vert^2 > \frac{D_0}{D_1}}}{\sqrt{\bnu_{T}}}\right],
\end{align}
\normalsize
where in the third inequality, we use the following fact, 
\begin{equation*}
    2D_1 \Vert \nabla f(\bw_t) \Vert^2\mathds{1}_{\Vert \nabla f(\bw_t)\Vert^2 > \frac{D_0}{D_1}}\ge (D_0+D_1 \Vert \nabla f(\bw_t) \Vert^2)  \mathds{1}_{\Vert \nabla f(\bw_t)\Vert^2 > \frac{D_0}{D_1}}\ge \mathbb{E}^{|\mathcal{F}_t} \Vert g_t \Vert^2 \mathds{1}_{\Vert \nabla f(\bw_t)\Vert^2> \frac{D_0}{D_1}}.
\end{equation*}
We then consider the iterations satisfying $\nabla f(\bw_t)\le\frac{D_0}{D_1}$,
\small
\begin{align}
\nonumber
    &\frac{1}{2D_1 }  \sum_{t=1}^T \mathbb{E}\left[\frac{\Vert g_t \Vert^2}{\sqrt{\bnu_{T}}}\mathds{1}_{\Vert \nabla f(\bw_t)\Vert^2 \le \frac{D_0}{D_1}}\right]+ \frac{1}{2D_1 }  \mathbb{E}\frac{\bnu_0}{\sqrt{\bnu_{T}}}
    \le   \frac{1}{2D_1 }   \mathbb{E}\left[\frac{\sum_{t=1}^T\Vert g_t \Vert^2\mathds{1}_{\Vert \nabla f(\bw_t)\Vert^2 \le \frac{D_0}{D_1}}+\bnu_0}{\sqrt{\sum_{t=1}^T \Vert g_t \Vert^2\mathds{1}_{\Vert \nabla f(\bw_t)\Vert^2 \le \frac{D_0}{D_1}}+\bnu_0}}\right] 
    \\
    \nonumber
    =&  \frac{1}{2D_1 }   \mathbb{E}\sqrt{\sum_{t=1}^T\Vert g_t \Vert^2\mathds{1}_{\Vert \nabla f(\bw_t)\Vert^2 \le \frac{D_0}{D_1}}+\bnu_0}
        \le  \frac{1}{2D_1 }  \sqrt{\mathbb{E}\left[\sum_{t=1}^T\Vert g_t \Vert^2\mathds{1}_{\Vert \nabla f(\bw_t)\Vert^2 \le \frac{D_0}{D_1}}\right]+\bnu_0} 
        \\
        \label{eq: case_small}
        \le&  \frac{1}{2D_1 } \sqrt{\mathbb{E}\left[\sum_{t=1}^T(D_1\Vert \nabla f(\bw_t)) \Vert^2 +D_0)\mathds{1}_{\Vert \nabla f(\bw_t)\Vert^2 \le \frac{D_0}{D_1}}+\bnu_0\right]}             
    \le \frac{1}{2D_1} \sqrt{2D_0T+\bnu_0}.
\end{align}
\normalsize
Here in the second inequality we use Jensen's inequality, and in the third we use Assumption \ref{assum: affine}.                                    
Putting Eq. (\ref{eq: case_large}) and Eq. (\ref{eq: case_small}) together, we then have
\small
\begin{align*}
    &\frac{1}{2D_1 } \mathbb{E}[\sqrt{\bnu_T}] = \frac{1}{2D_1 }  \mathbb{E}\left[\frac{\sum_{t=1}^T\Vert g_t \Vert^2 \mathds{1}_{\Vert \nabla f(\bw_t)\Vert^2 > \frac{D_0}{D_1}}}{\sqrt{\bnu_{T}}}\right]+
    \frac{1}{2D_1 }  \sum_{t=1}^T \mathbb{E}\left[\frac{\Vert g_t \Vert^2}{\sqrt{\bnu_{T}}}\mathds{1}_{\Vert \nabla f(\bw_t)\Vert^2 \le \frac{D_0}{D_1}}\right]+ \frac{1}{2D_1 } \mathbb{E}\frac{\bnu_0}{\sqrt{\bnu_{T}}}
    \\
   \le & \frac{1}{2D_1} \sqrt{2D_0T+\bnu_0} +C_1+C_2 \mathbb{E} \ln \bnu_T\le \frac{1}{2D_1} \sqrt{2D_0T+\bnu_0} +C_1+2C_2  \ln \mathbb{E} \sqrt{\bnu_T}.
\end{align*}
\normalsize
Here in the last inequality we use Jensen's inequality. Solving the above inequality with respect to $\mathbb{E}[\sqrt{\bnu_T}]$, we have 
$
   \mathbb{E}[\sqrt{\bnu_T}]\le 2\sqrt{2D_0T+\bnu_0}+4 D_1C_1 +48C_2 D_1 \ln  (4C_2 D_1+e)$.
As
       $$C_1 +2C_2 \ln \mathbb{E} \sqrt{\bnu_T} \ge   \sum_{t=1}^T \mathbb{E}\left[\frac{\Vert \nabla f(\bw_t) \Vert^2}{\sqrt{\bnu_{t-1}}}\right]\ge  \mathbb{E}\left[\frac{\sum_{t=1}^T\Vert \nabla f(\bw_t) \Vert^2}{\sqrt{\bnu_{T}}}\right]
     \ge\frac{\mathbb{E}\left[\sqrt{\sum_{t=1}^T\Vert \nabla f(\bw_t) \Vert^2}\right]^2}{\mathbb{E}[\sqrt{\bnu_T}]},$$
 where the last inequality is due to Hölder's inequality.
Applying the estimation of $\mathbb{E}\sqrt{\bnu_T}$, we obtain that $ \mathbb{E}\left[\sqrt{\sum_{t=1}^T\Vert \nabla f(\bw_t) \Vert^2}\right]^2\le (2\sqrt{2D_0T+\bnu_0}+C_3)(C_1+2C_2\ln(2\sqrt{2D_0T+\bnu_0}+C_3))$.

By further applying Markov's inequality, we conclude the proof. \QEDA
\begin{remark}
    The proof in \citep{faw2022power} can also be divided into two stages with similar goals. Our proof is simpler in both stages, and we discuss the reason here. As pointed out in Section \ref{sec: motivation}, our proof in Stage I is simpler is due to the novel auxiliary function $\xi$. Moreover, our conclusion in Stage I is also stronger, which laid a better foundation for Stage II: \cite{faw2022power} can only derive the bound 
 of $\mathbb{E}\sum_{t\in \tilde{S}}\frac{\Vert \nabla f(\bw_t)\Vert^2}{\sqrt{\tilde{\bnu}_{t}}}$, where $\tilde{S}$ is a subset of $[T]$. This raises additional challenges for Stage II in \citep{faw2022power}, as our divide-and-conquer technique can no longer be applied. \cite{faw2022power} resolve this through a recursively-improving technique, which not only require a complicated proof, but also entangles $D_0$ and $D_1$ and leads to a sub-optimal rate under strong growth condition.
\end{remark}

\section{Extending the analysis to AdaGrad}
\label{sec: adagrad}
In this section, we extend the convergence analysis of AdaGrad-Norm to AdaGrad. Such a result is attractive since AdaGrad is more commonly used in practice than AdaGrad-Norm. 
% This section is organized as follows: in Section \ref{sec: wr adagrad}, we prove that Theorem \ref{thm: adagrad_norm} can be extended to AdaGrad, but under a more restricted assumption called coordinate-wise affine noise variance assumption. Then, in Section \ref{sec: wor adagrad}, we prove that the coordinate-wise affine noise variance assumption can be loosened into affine noise variance assumption if we investigate randomly-shuffled AdaGrad with every batch-loss is smooth.
% \subsection{Convergence of AdaGrad under coordinate-wise affine noise variance assumption}
% \label{sec: wr adagrad}
A natural hope is to prove the convergence of AdaGrad under the same set of assumptions in Theorem \ref{thm: adagrad_norm}. However, challenge arises when we try to derive an AdaGrad version of Lemma \ref{lem: auxiliary}. Concretely, the ``First-Order Main" term becomes $\eta \sum_{l=1}^d \frac{\partial_i f(\bw_t)^2}{\sqrt{\bnu_{t-1,l}}}$ (we use $\bnu_{t,l}$ as the $l$-th coordinate of $\bnu_t$, and similar does $g_{t,l}$), while the bound of the ``Error" term  includes a term $\frac{\eta}{4} \sum_{l=1}^d \frac{\Vert \nabla  f(\bw_t) \Vert^2}{\sqrt{\bnu_{t-1,l}}}$ and thus can not be controlled. Such a mismatch is due to that $\mathbb{E}[g_{t,i}^2]$ can only be bounded by the full gradient  $\Vert \nabla f(\bw_t) \Vert $ instead of the partial derivative of the corresponding coordinate $\vert \partial_i f(\bw_t) \vert $ (see Appendix \ref{appen: explain} for details). Therefore, to derive the convergence of AdaGrad, we strengthen Assumption \ref{assum: affine} to let the affine noise variance hold coordinate-wisely.
\begin{assumption}[Coordinate-wise affine noise variance assumption]
\label{assum: coordinate_affine}
We assume that there exist positive constants $D_0$ and $D_1$ such that $\forall t \ge 1$ and $\forall i \in [d]$,
    $\mathbb{E}[\vert g_{t,i}\vert^2| \mathcal{F}_t] \le D_{0}+D_{1} \partial_i f(\bw_t)^2$.
\end{assumption}
Note that Assumption \ref{assum: coordinate_affine} is \emph{still general than most of the assumptions in existing works}. As an example, the bounded noise variance assumption is its special case. Next, we obtain the convergence  result for AdaGrad as follows.

\begin{theorem}[AdaGrad]
\label{thm:adagrad}
     Let Assumptions \ref{assum: smooth} and \ref{assum: coordinate_affine} hold. Then, for AdaGrad with any  learning rate $ \eta >0$, we have that with probability at least $1-\delta$, $\min_{t\in [T]} \Vert \nabla f(\bw_t) \Vert^2=\mathcal{O}(\frac{1+\ln (1+\sqrt{D_0T})}{T\delta^2})+\mathcal{O}(\frac{\sqrt{D_0}(1+\ln (1+\sqrt{D_0T}))}{\sqrt{T}\delta^2})$.
\end{theorem}
The proof is a coordinate-wise version of the proof for Theorem \ref{thm: adagrad_norm} with some modifications, where we leverage a coordinate-wise version of $\xi(t)$, i.e., $\tilde{\xi}(t)=\sum_{l=1}^d \frac{\partial_l f(\bw_t)^2}{\sqrt{\bnu_{t,l}}}$ (please refer to Appendix \ref{appen: adagrad_wr} for details).

% \subsection{Convergence of random-reshuffling AdaGrad under affine noise variance assumption}
% \label{sec: wor adagrad}
We still seek to relax Assumption \ref{assum: coordinate_affine} back to Assumption \ref{assum: affine}. This is because  Assumption \ref{assum: coordinate_affine} may preclude some basic objectives. We demonstrate this idea through the following example.

\begin{example}
 Consider the following linear regression problem: $f(\bw)= \mathbb{E}_{\bx\sim \mathcal{D}} (\langle \bw, \bx\rangle )^2=\Vert \bw\Vert^2$, where $\mathcal{D}$ is a standard Gaussian distribution over $\mathbb{R}^d$ with the absolute value of each coordinate truncated by $1$.  At point $\bw$, define the stochastic gradient $g(\bw)$ as $2\bx \bx^\top \bw$, where $\bx$ is sampled according to $\mathcal{D}$. One can easily verify that $g(\bw)$ is an unbiased estimation of $\nabla f(\bw)$. For this example, $\mathbb{E}_{\bx\sim \mathcal{D}} \Vert g(\bw) \Vert^2=\Theta(\Vert \bw\Vert^2)$ and $\Vert \nabla f(\bw) \Vert^2=\Theta(\Vert \bw\Vert^2)$. Therefore, Assumption \ref{assum: affine} holds with $L_0= 0$. However, $\Vert \partial_1 f(\bw) \Vert^2=4 (\bw)_1^2$ and $\mathbb{E}_{\bx\sim \mathcal{D}}  (g(\bw))_1 ^2=\Theta(\Vert \bw\Vert^2)$, we see that $\lim_{\vert(\bw)_{2}\vert \rightarrow \infty} \frac{\Vert \partial_1 f(\bw) \Vert^2}{\mathbb{E}_{\bx\sim \mathcal{D}}  (g(\bw))_1 ^2}=\infty$, which violates Assumption \ref{assum: coordinate_affine}.
\end{example}

\begin{wrapfigure}{L}{0.5\textwidth}
\begin{minipage}{0.5\textwidth}
\begin{algorithm}[H]
\label{alg: rr_adagrad}
\caption{Randomly-reshuffled AdaGrad}
\hspace*{0.02in} {\bf Input:}
Objective function $f(\bw):=\frac{1}{n}\sum_{i=1}^n f_i(\bw)$, learning rate $\eta>0$, $\bw_{1,1} \in \mathbb{R}^d$, $\bnu_{1,0}\in \mathbb{R}^{d,+}$
\begin{algorithmic}[1]
\State \textbf{For} $t=1\rightarrow \infty$:
\State ~~~~~~\textbf{For} $i=1\rightarrow n$:
\State ~~~~~~~~~~~~Uniformly sample $\{\tau_{t,1},\cdots,\tau_{t,n}\}$ as \text{~~~~~~~~~~~~~a} random permutation of $[n]$
\State ~~~~~~~~~~~~Calculate $g_{t,i}=\nabla f_{\tau_{t,i}} (\bw_{\tau_{t,i}})$
\State ~~~~~~~~~~~~Update $\bnu_{t,i}=\bnu_{t,i-1}+g_{t,i}^{\odot2}$
\State ~~~~~~~~~~~~Update $\bw_{t,i+1}= \bw_{t,i}-\eta \frac{1}{\sqrt{\bnu_{t,i}}}\odot  g_{k,i} $
\State ~~~~~~\textbf{EndFor}
\State ~~~~~~ Update $\bw_{t+1,1}= \bw_{t,n+1}$, $\bnu_{t+1,0}= \bnu_{t,n}$
\State \textbf{EndFor}
\end{algorithmic}
\end{algorithm}
\end{minipage}
\vspace{-5mm}
\end{wrapfigure}
On the other hand, note that the above example obeys a stronger assumption on the smoothness, i.e, for every fixed $x$, the stochastic gradient $g(\bw)$ is globally Lipschitz. It is natural to ask whether we can relax Assumption \ref{assum: coordinate_affine} by strengthening the assumption on the smoothness. In Section \ref{sec: motivation}, we explain that we use $\frac{  \Vert \nabla f(\bw_t) \Vert^2}{\sqrt{\bnu_{t-1}}}$ instead of $\frac{\Vert g_t\Vert  \Vert \nabla f(\bw_t) \Vert}{\sqrt{\bnu_{t-1}}}$ as the auxiliary function due to $g_t \not \approx g_{t-1}$. Therefore, tightening the assumption on the smoothness may help us to ensure $g_t \approx g_{t-1}$, and we no longer need to bound $g_t$ using Assumption \ref{assum: affine}, and as a result we will not encounter the mismatch between $\frac{\partial_i f(\bw_t)^2}{\sqrt{\bnu_{t-1,i}}}$ and $\frac{\Vert \nabla  f(\bw_t) \Vert^2}{\sqrt{\bnu_{t-1,i}}}$. Motivated by the above example, we make the assumption that we have the access to an stochastic oracle $g(\bw,\zeta)$, and $g_t$ is generated by $g_t = g(\bw_t,\zeta_t)$ where $\zeta_t$ is sampled independently from a distribution $\mathcal{D}$. We further assume that $g$ is $L$ Lipschitz with respect to $\bw$ for a fixed $\zeta$. Such an assumption is common in stochastic optimization literature \citep{yun2021minibatch,shi2021rmsprop}.
 Unfortunately, such assumptions are still not adequate to ensure $g_t\approx g_{t-1}$ since $g_t$ and $g_{t-1}$ may use different noise $\zeta$. A good news is that, for the without-replacement version of AdaGrad (also called randomly-reshuffled AdaGrad. Please refer to Algorithm \ref{alg: rr_adagrad}), every $\zeta$ appears once within one epoch and the above methodology can be used. Note that randomly-reshuffled AdaGrad is the version of AdaGrad commonly adopted in deep learning. Thus although we slightly change the analyzed algorithm, the problem we consider is still of significance.

As mentioned above, we require the following assumptions for the convergence of randomly-reshuffled AdaGrad.
\begin{assumption}[Assumption \ref{assum: affine}, reformulated]
\label{assum: rr_affine}
 Let $\bw_{k,i}$ and $g_{k,i}$ be the ones in Algorithm \ref{alg: rr_adagrad}. Then, there exist constants $D_0$ and $D_1$, such that, $\forall k,i $, $
     \mathbb{E}_{j\sim \textbf{Uniform} (n)} \Vert \nabla f_j(\bw_{k,i}) \Vert^2 \le D_0+D_1 \Vert \nabla f(\bw_{k,i}) \Vert^2$.
\end{assumption}

\begin{assumption}[Stochastic $L$-smooth condition]
\label{assum: stochatic_smooth}
We assume that $\forall i\in [n]$, $f_i$ is differentiable and its gradient
satisfies $\forall \bw_1,\bw_2 \in \mathbb{R}^d$, we have $
    \Vert \nabla f_i(\bw_1)-\nabla f_i(\bw_2) \Vert \le L \Vert \bw_1-\bw_2 \Vert$.
\end{assumption}

\begin{theorem}[Randomly-reshuffled AdaGrad]
\label{thm: rr_adagrad}
     Let Assumptions \ref{assum: rr_affine} and \ref{assum: stochatic_smooth} hold. Then, for randomly-reshuffled AdaGrad with any $ \eta >0$, $\min_{t\in [T]} \Vert \nabla f(\bw_{t,0}) \Vert^2=\mathcal{O}(\frac{1+\ln (1+\sqrt{D_0T})}{T})+\mathcal{O}(\frac{\sqrt{D_0}(1+\ln (1+\sqrt{D_0T}))}{\sqrt{T}})$.
\end{theorem}
The proof utilizes a randomly-reshuffling version of $\xi(t)$, i.e. $\bar{\xi}(t)=\sum_{i=1}^n\sum_{l=1}^d\frac{\vert\partial_l f(\bw_{t,1})\vert  \vert \partial_l f_{\tau_{t,i}}(\bw_{t,i})\vert}{\sqrt{\bnu_{t,i,l}}}$, and we defer the details to Appendix \ref{appen: adagrad_wor}. Theorem \ref{thm: rr_adagrad} shows that randomly-reshuffled AdaGrad does converge under affine noise variance assumption, and extends our analysis techniques to new setting of AdaGrad.

\section{Convergence of AdaGrad over non-uniformly smooth landscapes}
\label{sec: non-uniform}
So far, the characterizations of AdaGrad(-Norm) has closely matched those of SGD over the uniformly smooth landscape. However, in practice, the objective function is usually non-uniformly smooth. Simple examples include polynomial functions with degree larger than $2$, and deep neural networks. A natural question is that whether AdaGrad still works well over non-smooth landscapes. In this section, we analyze AdaGrad(-Norm) under the $(L_0,L_1)$-smooth condition \citet{zhang2019gradient}, which is considered as a preciser characterization of the landscape of neural networks through exhaustive experiments.

\begin{assumption}[$(L_0,L_1)$-smooth condition]
\label{assum: non-smooth}
We assume that $f: \mathbb{R}^d \rightarrow \mathbb{R}$ is differentiable. We further assume that there exists positive constants $L_0$ and $L_1$, such that if $\bw_1,\bw_2 \in \mathbb{R}^d$ satisfies $\Vert \bw_1-\bw_2 \Vert \le \frac{1}{L_0}$, then $\Vert  \nabla f(\bw_1) - \nabla f(\bw_2) \Vert \le (L_1 \Vert \nabla f(\bw_1)\Vert +L_0) \Vert \bw_1-\bw_2 \Vert$.
\end{assumption}

Assumption \ref{assum: non-smooth} degenerates to Assumption \ref{assum: smooth} with $L=L_0$ if $L_1=0$. Therefore, Assumption \ref{assum: non-smooth} is more general than Assumption \ref{assum: smooth}. Moreover, Assumption \ref{assum: non-smooth} holds for polynomials of any degree and even exponential functions. \citet{zhang2019gradient} demonstrate through extensive experiments that over the tasks where adaptive optimizers outperforms SGD, Assumption \ref{assum: non-smooth} is obeyed.

\begin{theorem}
\label{thm: adagrad_norm_nonsmooth}
    Let Assumptions \ref{assum: affine} and \ref{assum: non-smooth} hold. Then, for AdaGrad-Norm with $\forall \eta< \frac{1}{L_1}\min\{\frac{1}{64D_1},\frac{1}{8\sqrt{D_1}}\} $, we have that with probability at least $1-\delta$, \begin{equation*}
        \min_{t\in [T]} \Vert \nabla f(\bw_t) \Vert^2=\mathcal{O}\left(\frac{1+\ln (1+\sqrt{D_0T})}{T\delta^2}\right)+\mathcal{O}\left(\frac{\sqrt{D_0}(1+\ln (1+\sqrt{D_0T}))}{\sqrt{T}\delta^2}\right).
    \end{equation*}
\end{theorem}

The proof can be found in Appendix \ref{appen: adagrad_ns}, the key insight of which is that the additional error terms caused by $(L_0,L_1)$-smooth condition is at the same order of the ``Error" term in the proof of Theorem \ref{thm: adagrad_norm}. Theorem \ref{thm: adagrad_norm_nonsmooth} shows that AdaGrad-Norm can provably overcome the non-uniform smoothness in the objective function and converges. Similar result can be derived for AdaGrad.  Compared to Theorem \ref{thm: adagrad_norm}, we additionally require the learning rate to be lower than a threshold. To see such a requirement is not artificial due to  the proof, we provide the following theorem.

\begin{theorem}
\label{thm: counter}
    For every learning rate $\eta> \frac{9\sqrt{5}}{2L_1}$, there exist a lower-bounded objective function $f$ obeying Assumption \ref{assum: non-smooth}  and a corresponding initialization point $\bw_0$, such that AdaGrad with learning rate $\eta$ and initialized at $\bw_0$ diverges over $f$.
\end{theorem}

The proof can be found in Appendix \ref{appen: counter}. Theorem \ref{thm: counter} shows that the learning rate requirement in Theorem \ref{thm: adagrad_norm_nonsmooth} is tight w.r.t. $L_1$ up to constants. The tuning-free ability of AdaGrad under $L$-smooth condition, i.e, AdaGrad converges under any learning rate, has been considered to be superiority of AdaGrad over SGD. However, combining Theorem \ref{thm: adagrad_norm_nonsmooth} and Theorem \ref{thm: counter}, we demonstrate that such a property is lost when under a more realistic assumption on the smoothness. On the other hand, \citet{zhang2019gradient} show that SGD converges arbitrarily slowly under $(L_0,L_1)$-smooth condition. Together with Theorem \ref{thm: adagrad_norm_nonsmooth}, we find another superiority of AdaGrad: it can provably overcome non-uniform smoothness but SGD can not.

\section{Related works}\
\label{sec: related}
\textbf{Convergence of AdaGrad over non-convex landscapes.} \citet{duchi2011adaptive} and \citet{mcmahan2010adaptive} simultaneously propose AdaGrad. Since then, there is a line of works analyzing the convergence of AdaGrad over non-convex landscapes \citep{ward2020adagrad,li2019convergence,zou2018weighted,li2020high,defossez2020simple,gadat2020asymptotic,kavis2022high,faw2022power,yang2022nest,yang2023two}. We summarize their assumptions and conclusions in Table \ref{tab: related works}.
\\
\textbf{Non-uniform smoothness.} The convergence analysis of optimizers under non-uniform smoothness is initialized by \citep{zhang2019gradient}, who propose the $(L_0,L_1)$-smooth condition and verify its validity in deep learning. They further prove that Clipped SGD converges under such a condition. Since then, their analysis has been extended to other clipped optimizers \citep{zhang2020improved,yang2022normalized,crawshaw2022robustness}. However, there is no such a result for AdaGrad.

\section{Conclusion}
In this paper, we analyze AdaGrad over non-convex landscapes. Specifically, we propose a novel auxiliary function to bound the error term that is brought by the update of AdaGrad. Based on this auxiliary function, we are able to significantly simplify the proof of AdaGrad-Norm and establish a tighter convergence rate in the over-parameterized regime. We further extend the analysis to AdaGrad and non-uniformly smooth landscapes through different variants of the auxiliary function. One future direction is to explore and compare the convergences of AdaGrad and Adam under the $(L_0, L_1)$-smooth condition, given the fact that convergences of AdaGard and SGD are clearly separated under this condition.

\acks{This work is founded by CAS Project for Young Scientists in Basic Research under Grant No. YSBR-034, Innovation Project of ICT CAS under Grants No. E261090. The authors would also like to thank Mr. Matthew Faw for the helpful discussions.}

%As for future directions, we cast the following questions:
% \begin{itemize}
%     \item So far, the convergence rate of AdaGrad is no better than AdaGrad-Norm yet requiring stronger assumptions, despite that AdaGrad is more often adopted in practice. What is the theoretical benefit of coordinate-wise learning rate?

%     \item Convergence over $(L_0,L_1)$-smooth condition separates AdaGrad and SGD. How to  
% compare AdaGrad with other adaptive optimizers, e.g., Adam?
% \end{itemize}
\bibliography{Related}

\newpage

\appendix

\section{Table listing existing works on convergence of AdaGrad over non-convex landscapes}
\begin{table*}[htb!]
\resizebox{\textwidth}{!}{
\begin{tabular}{c|c|c|c|c}
 & Allow unbounded gradient&  Not delayed$^{(a)}$  & Tuning-free learning rate  &  Additional comments  \\ \hline
\citet{ward2020adagrad} & $\times$& $\checkmark$& $\checkmark$ & AdaGrad-Norm
\\
\citet{li2019convergence} & $\checkmark$& $\times^{(b)}$&$\times$& AdaGrad-Norm, sub-gaussian noise 
\\
\citet{zou2018weighted} & $\times$& $\checkmark$ & $\checkmark$ & AdaGrad
\\
\citet{li2019convergence} & $\checkmark$& $\times$&$\times$& AdaGrad, sub-gaussian noise 
\\
\citet{defossez2020simple}& $\times$ & $\checkmark$& $\checkmark$ & AdaGrad
\\
\citet{gadat2020asymptotic} & $\times$ & $\checkmark$ & $ \times$ & AdaGrad, asymptotic
\\
\citet{kavis2022high} & $\times$ & $\checkmark$ & $ \checkmark$ & AdaGrad-Norm$^{(c)}$
\\
\citet{yang2022nest,yang2023two}$^{d}$ & $\checkmark$& $\checkmark$ & $\checkmark$ &AdaGrad-Norm 
\\
\citet{faw2022power} & $\checkmark$& $\checkmark$ & $\checkmark$ &AdaGrad-Norm 
\end{tabular}
}
\small
\caption{ Summary of existing convergence result of AdaGrad. \textbf{None of these works show that AdaGrad converges at rate $\frac{1}{T} \Vert \nabla f(\bw_t) \Vert^2=\mathcal{O}(1/T)$ when over-parameterized.} 
We provide some explanation on the upper footmarks: $(a)$: "delayed" refers to Delayed AdaGrad, where $\bnu_t$ does not contain the information of $g_t$;  $(b)$: \citet{li2019convergence} also change the degree of $\bnu_t$ in the adaptive learning rate from $-\frac{1}{2}$ to $-(\frac{1}{2}+\varepsilon)$; $(c)$: the bound in \citep{kavis2022high} has logarithmic dependence on the probability margin; $(d)$: these two works assume bounded noise variance, which is stronger than affine noise variance.} 
\normalsize
\label{tab: related works}
\end{table*}

\section{Preparations}
\label{appen: preparations}
This section collects technical lemmas that will be used latter.

\begin{lemma}
\label{lem: summation_of_series}
    Let $\{a_i\}_{i=0}^{\infty}$ be a series of non-negative real numbers with $a_0>0$. Then, the following inequalities hold:
    \begin{gather*}
        \sum_{t=1}^T \frac{a_t}{\sqrt{(\sum_{s=0}^t a_s)^3}} \le 2\frac{1}{\sqrt a_0},~~
        \sum_{t=1}^T \frac{a_t}{\sum_{s=0}^t a_s} \le  \ln \sum_{t=0}^T a_t -\ln a_0,
        \\
        \sum_{t=1}^T \frac{a_t}{\sqrt{\sum_{s=0}^t a_s}(\sqrt{\sum_{s=0}^{t-1} a_s}+\sqrt{\sum_{s=0}^t a_s})^2} \le \frac{1}{\sqrt a_0}.
    \end{gather*}
\end{lemma}

    \begin{proof}[Proof of Lemma \ref{lem: summation_of_series}]
    The first inequality is because $ \sum_{t=1}^T \frac{a_t}{\sqrt{(\sum_{s=0}^t a_s)^3}}\le  \int_{x=a_0}^{\sum_{t=0}^T a_t}  \frac{1}{\sqrt{x}^3} \mathrm{d} x $. The second inequality is because $ \sum_{t=1}^T \frac{a_t}{\sum_{s=0}^t a_s}\le  \int_{x=a_0}^{\sum_{t=0}^T a_t}  \frac{1}{x} \mathrm{d} x $.
    The third inequality is due to 
    \small
    \begin{align*}
       &\frac{a_t}{\sqrt{\sum_{s=0}^t a_s}(\sqrt{\sum_{s=0}^{t-1} a_s}+\sqrt{\sum_{s=0}^t a_s})^2}\le  \frac{a_t}{\sqrt{\sum_{s=0}^t a_s}\sqrt{\sum_{s=0}^{t-1} a_s}(\sqrt{\sum_{s=0}^{t-1} a_s}+\sqrt{\sum_{s=0}^{t} a_s})}
       \\
=&\frac{1}{\sqrt{\sum_{s=0}^{t-1} a_s}}-\frac{1}{\sqrt{\sum_{s=0}^{t} a_s}}.
    \end{align*}
    \normalsize
    The proof is completed.
\end{proof}

\section{Analysis of AdaGrad}
\label{appen: adagrad}
\subsection{Proof of Theorem \ref{thm:adagrad}}
Before the proof, we define $g_{k,l}$ as the $l$-th component of $g_{k}$ and $\bnu_{k,l}$ as the $l$-th component of $\bnu_k$.

\begin{proof}[Proof of Theorem \ref{thm:adagrad}]
\label{appen: adagrad_wr}
As the proof highly resembles that of Theorem \ref{thm:adagrad}, we only highlight several key steps with the rest of the details omitted.

By the descent lemma, we have that 
\small
\begin{align}
\nonumber\mathbb{E}^{|\mathcal{F}_t}f(\bw_{t+1})\le&  f(\bw_{t}) + \begin{matrix}\underbrace{-\eta\left\langle \frac{1 }{\sqrt{\bnu_{t-1}} } \odot \nabla f(\bw_t), \nabla f(\bw_t) \right\rangle }
   \\
    \text{First Order Main}
    \end{matrix} + \begin{matrix}\underbrace{\mathbb{E}^{|\mathcal{F}_t}\left[\left\langle \nabla f(\bw_t), \eta \left(\frac{1}{\sqrt{\bnu_{t-1}} }-\frac{1}{\sqrt{\bnu_{t}} }\right)\odot g_t \right\rangle\right] }
   \\
    \text{Error}
    \end{matrix} 
    \\
    \label{eq: descent_adagrad_wr}
    &+ \begin{matrix}\underbrace{\frac{L}{2}\eta^2\mathbb{E}^{|\mathcal{F}_t} \left\Vert  \frac{1}{\sqrt{\bnu_t} } \odot g_t \right\Vert^2}
   \\
    \text{Second Order}
   \end{matrix} .
\end{align}
\normalsize
The "Error" term can be expanded as 
\begin{align*}
\mathbb{E}^{|\mathcal{F}_t}\left[\left\langle \nabla f(\bw_t), \eta \left(\frac{1}{\sqrt{\bnu_{t-1}} }-\frac{1}{\sqrt{\bnu_{t}} }\right)\odot g_t \right\rangle\right]= \sum_{l=1}^d\mathbb{E}^{|\mathcal{F}_t}\left[ \eta \partial_l f(\bw_t)  \left(\frac{1}{\sqrt{\bnu_{t-1,l}} }-\frac{1}{\sqrt{\bnu_{t,l}} }\right) g_{t,l} \right].
\end{align*}

For each $l$, the RHS of the above inequality can then be bounded as 
\begin{equation*}
   \mathbb{E}^{|\mathcal{F}_t}\left[ \eta \partial_l f(\bw_t)  \left(\frac{1}{\sqrt{\bnu_{t-1,l}} }-\frac{1}{\sqrt{\bnu_{t,l}} }\right) g_{t,l} \right] \le \frac{1}{2}\eta\frac{\partial_l f(\bw_t)^2}{\sqrt{\bnu_{t-1,l}} }+\frac{1}{2}\frac{\eta}{\sqrt{\bnu_{t-1,l}} }\left(\mathbb{E}^{|\mathcal{F}_t}\left[  \frac{ g_{t,l}^2}{\sqrt{\bnu_{t,l}}+\sqrt{\bnu_{t-1,l}}}\right]\right)^2,
\end{equation*}
where by Assumption \ref{assum: coordinate_affine}, $\left(\mathbb{E}^{|\mathcal{F}_t}\left[  \frac{ g_{t,l}^2}{\sqrt{\bnu_{t,l}}+\sqrt{\bnu_{t-1,l}}}\right]\right)^2$ can be bounded by
\begin{equation*}
   \left(\mathbb{E}^{|\mathcal{F}_t}\left[  \frac{ g_{t,l}^2}{\sqrt{\bnu_{t,l}}+\sqrt{\bnu_{t-1,l}}}\right]\right)^2\le ( D_0+ D_1 \partial_l f(\bw_{t})^2)  \mathbb{E}^{|\mathcal{F}_t}\left[  \frac{ g_{t,l}^2}{(\sqrt{\bnu_{t,l}}+\sqrt{\bnu_{t-1,l}})^2}\right].
\end{equation*}

Then, following the similar routine as the proof of Theorem \ref{thm: adagrad_norm} and leveraging the potential function $\tilde{\xi}_l (t) \triangleq \frac{\partial_l f(\bw_t)^2}{\sqrt{\bnu_{t,l}}}$, we have
\begin{align*}
&\mathbb{E}^{|\mathcal{F}_t}\left[ \eta \partial_l f(\bw_t)  \left(\frac{1}{\sqrt{\bnu_{t-1,l}} }-\frac{1}{\sqrt{\bnu_{t,l}} }\right) g_{t,l} \right]\le \frac{3}{4}\eta\frac{\vert \partial_l f(\bw_t)\vert^2}{\sqrt{\bnu_{t-1,l}} }+\frac{1}{2}\frac{\eta}{\sqrt{\bnu_{t-1,l}} }  D_0 \mathbb{E}^{|\mathcal{F}_t}\left[  \frac{\Vert g_{t,l}\Vert^2}{(\sqrt{\bnu_{t,l}}+\sqrt{\bnu_{t-1,l}})^2}\right]
    \\
    &+\frac{\eta}{2} D_1  \mathbb{E}^{|\mathcal{F}_t}\left(\tilde{\xi}_l(t-1)-\tilde{\xi}_l(t)\right)
    +\left(\eta (L \eta D_1)^2 +\frac{\eta}{2} D_1 (L \eta)^2\right)  \frac{1}{\sqrt{\bnu_{t-1,l}}} \left\Vert \frac{1}{\sqrt {\bnu_{t-1}}} \odot g_{t-1} \right\Vert^2 
\end{align*}
Applying the above inequality back to Eq. (\ref{eq: descent_adagrad_wr}) and summing Eq. (\ref{eq: descent_adagrad_wr}) over $t$ from $1$ to $T$, we obtain that
\begin{align*}
    &\frac{1}{4} \eta \sum_{t=1}^T \sum_{l=1}^d \mathbb{E}\frac{\partial_l f(\bw_t)^2}{\sqrt{\bnu_{t-1,l}}}\le f(\bw_1)-\mathbb{E}[f(\bw_T)]+\sum_{l=1}^d \frac{\eta D_1}{2} \mathbb{E}[\tilde{\xi}_l(0)-\tilde{\xi}_l(t)]+\frac{\eta}{2} D_0 \frac{1}{\sqrt{\bnu_0}}
    \\
     &~~~~~~~~~~~~~~~~~~~~~~~~~~+\left(  \frac{1}{\min_l\sqrt{\bnu_{0,l}}}\left(\eta (L \eta D_1)^2 +\frac{\eta}{2} D_1 (L \eta)^2\right) +\frac{L}{2}\eta^2\right) (\mathbb{E}\ln \bnu_T-\ln \bnu_0).
\end{align*}
Define
\small
\begin{gather*}
\tilde{C}_2 \triangleq 4\left(  \frac{1}{\min_l\sqrt{\bnu_{0,l}}}\left(\eta (L \eta D_1)^2 +\frac{\eta}{2} D_1 (L \eta)^2\right) +\frac{L}{2}\eta\right),
\\
\tilde{C}_1 \triangleq \frac{4}{\eta}f(\bw_1)-\frac{4}{\eta}\mathbb{E}[f(\bw_T)]+\sum_{l=1}^d 2 D_1 \mathbb{E}[\tilde{\xi}_l(0)-\tilde{\xi}_l(T)]+2   D_0 \frac{1}{\sqrt{\bnu_0}}- \tilde{C}_2\ln \bnu_0.
\end{gather*}
\normalsize

Then, following the same routine as Stage II in the proof of Theorem \ref{thm: adagrad_norm}, we have 
\begin{align*}
    \frac{1}{2D_1} \mathbb{E}\left[\sum_{l=1}^d \sqrt{\bnu_{T,l}}\right]\le& \frac{1}{2D_1} \sum_{l=1}^d\sqrt{2D_0T+ \bnu_{0,l}}+\tilde{C}_1+\tilde{C}_2 \mathbb{E} \sum_{l=1}^d\ln \bnu_{T,l}
    \\
    \le & \frac{1}{2D_1} \sum_{l=1}^d\sqrt{2D_0T+ \bnu_{0,l}}+\tilde{C}_1+2d\tilde{C}_2  \ln \mathbb{E}\sum_{l=1}^d \sqrt{\bnu_{T,l}}.
\end{align*}

The rest of the proof flows just as Theorem \ref{thm:adagrad}. The proof is completed. 
\end{proof}

\subsection{Explanation of why Theorem \ref{thm:adagrad} uses Assumption \ref{assum: coordinate_affine}}
\label{appen: explain}
Here we provide a more detailed explanation  
 of why Theorem \ref{thm:adagrad} can not be established using Assumption \ref{assum: coordinate_affine}.

In the proof of the previous section, we see that by  Assumption \ref{assum: coordinate_affine},
    \begin{equation*}
   \left(\mathbb{E}^{|\mathcal{F}_t}\left[  \frac{ g_{t,l}^2}{\sqrt{\bnu_{t,l}}+\sqrt{\bnu_{t-1,l}}}\right]\right)^2\le ( D_0+ D_1 \partial_l f(\bw_{t})^2)  \mathbb{E}^{|\mathcal{F}_t}\left[  \frac{ g_{t,l}^2}{(\sqrt{\bnu_{t,l}}+\sqrt{\bnu_{t-1,l}})^2}\right].
\end{equation*}
The term $\partial_l f(\bw_{t})^2$ latter results in a $\frac{1}{4}\eta \frac{\partial_l f(\bw_{t})^2}{\sqrt{\bnu_{t-1,l}}}$ term in the bound of the "Error" term. However, if we replace Assumption \ref{assum: coordinate_affine} by Assumption \ref{assum: affine}, the bound of $\left(\mathbb{E}^{|\mathcal{F}_t}\left[  \frac{ g_{t,l}^2}{\sqrt{\bnu_{t,l}}+\sqrt{\bnu_{t-1,l}}}\right]\right)^2$ change to
 \begin{equation*}
   \left(\mathbb{E}^{|\mathcal{F}_t}\left[  \frac{ g_{t,l}^2}{\sqrt{\bnu_{t,l}}+\sqrt{\bnu_{t-1,l}}}\right]\right)^2\le ( D_0+ D_1 \Vert \nabla  f(\bw_{t})\Vert ^2)  \mathbb{E}^{|\mathcal{F}_t}\left[  \frac{ g_{t,l}^2}{(\sqrt{\bnu_{t,l}}+\sqrt{\bnu_{t-1,l}})^2}\right],
\end{equation*}
and the  term $\frac{1}{4}\eta \frac{\partial_l f(\bw_{t})^2}{\sqrt{\bnu_{t-1,l}}}$ changes to $\frac{1}{4}\eta \frac{\Vert \nabla  f(\bw_{t})\Vert ^2}{\sqrt{\bnu_{t-1,l}}}$ correspondingly. Such a term is no longer smaller than the absolute value of the "First Order Main" term $-\eta\sum_{l=1}^d \frac{\partial_l f(\bw_{t})^2}{\sqrt{\bnu_{t-1,l}}}$, and the proof can not go on.

\subsection{Proof of Theorem \ref{thm: rr_adagrad}}
\label{appen: adagrad_wor}
Before the proof, we define several notations. We define $\bnu_{k,i,l}$ as the $l$-th component of $\bnu_{k,i}$, $\forall k\ge 1, i \in [n], l\in [d]$. Similarly, we define $g_{k,i,l}$ as the $l$-th component of $g_{k,i}$. We define $\tilde{\tau}_{k,i}$ as the inverse index of $\tau_{k,i}$, in the sense that $\tau_{k,\tilde{\tau}_{k,i}}=i$, $\forall k\ge 1, i \in [n]$. One can easily check that $\{\tilde{\tau}_{k,i}\}_{i=1}^n=[n]$, $\forall k\ge 1$.

\begin{proof}
By Assumption \ref{assum: stochatic_smooth} and the triangle inequality, $f$ satisfies $L$-smooth condition. By applying the descent lemma to the change of parameters within a whole epoch, we have
\begin{align}
\label{eq: descent_adagrad}
    f(\bw_{k+1,1})\le f(\bw_{k,1}) + \langle \nabla f(\bw_{k,1}), \bw_{k+1,1}-\bw_{k,1} \rangle +\frac{L}{2} \Vert \bw_{k+1,1}-\bw_{k,1} \Vert^2. 
\end{align}

We temporarily focus on the first order term $\langle \nabla f(\bw_{k,1}), \bw_{k+1,1}-\bw_{k,1} \rangle$ as the second order term is as simple to be bounded as those in the proofs of Theorem \ref{thm: adagrad_norm} and Theorem \ref{thm:adagrad}. The $ \bw_{k+1,1}-\bw_{k,1}$ in the first order term can be rewritten as
\begin{align*}
    &\bw_{k+1,1}-\bw_{k,1} 
     = -\eta \sum_{i=1}^n\frac{1}{\sqrt{\bnu_{k,i}}} \odot g_{k,i}
    \\
    =&-\eta \sum_{i=1}^n\frac{1}{\sqrt{\bnu_{k,1}}} \odot g_{k,i}+ \eta \sum_{i=1}^n(\frac{1}{\sqrt{\bnu_{k,1}}} -\frac{1}{\sqrt{\bnu_{k,i}}}) \odot g_{k,i}
    \\
    =&-\eta \sum_{i=1}^n\frac{1}{\sqrt{\bnu_{k,1}}} \odot \nabla f_{\tau_{k,i}} (\bw_{k,1})+ \eta \sum_{i=1}^n(\frac{1}{\sqrt{\bnu_{k,1}}} -\frac{1}{\sqrt{\bnu_{k,i}}}) \odot g_{k,i}+\eta \sum_{i=1}^n\frac{1}{\sqrt{\bnu_{k,1}}} \odot (\nabla f_{\tau_{k,i}} (\bw_{k,1})-g_{k,i})
    \\
    =& -\eta \frac{n}{\sqrt{\bnu_{k,1}}} \odot \nabla f (\bw_{k,1})+ \eta \sum_{i=1}^n(\frac{1}{\sqrt{\bnu_{k,1}}} -\frac{1}{\sqrt{\bnu_{k,i}}}) \odot g_{k,i}+\eta \sum_{i=1}^n\frac{1}{\sqrt{\bnu_{k,1}}} \odot (\nabla f_{\tau_{k,i}} (\bw_{k,1})-g_{k,i})
\end{align*}
% where the last term above can be further bounded according to Assumption \ref{assum: coordinate_affine} as 
% \begin{align*}
%     &\Vert \sum_{i=1}^n\frac{1}{\sqrt{\bnu_{k,1}}} \odot (\nabla f_{\tau_{k,i}} (\bw_{k,0})-g_{k,i})\Vert 
%     \\
%     = & \Vert \sum_{i=1}^n\frac{1}{\sqrt{\bnu_{k,1}}} \odot (\nabla f_{\tau_{k,i}} (\bw_{k,0})-\nabla f_{\tau_{k,i}} (\bw_{\tau_{k,i}})\Vert 
% \end{align*}

Thus, the first order term can be rewritten as
\begin{align*}
    &\langle \nabla f(\bw_{k,1}) , \bw_{k+1}-\bw_{k,1} \rangle
    \\
    =&  \langle \nabla f(\bw_{k,1}), -\eta \frac{n}{\sqrt{\bnu_{k,1}}} \odot \nabla f (\bw_{k,1})\rangle + \langle \nabla f(\bw_{k,1}), \eta \sum_{i=1}^n(\frac{1}{\sqrt{\bnu_{k,1}}} -\frac{1}{\sqrt{\bnu_{k,i}}}) \odot g_{k,i}\rangle 
    \\
    &+\langle \nabla f(\bw_{k,1}),\eta \sum_{i=1}^n\frac{1}{\sqrt{\bnu_{k,1}}} \odot (\nabla f_{\tau_{k,i}} (\bw_{k,1})-g_{k,i})\rangle
    \\
    =& -\eta n\left\Vert  \frac{1}{\sqrt[4]{\bnu_{k,1}}} \odot \nabla f (\bw_{k,1})\right\Vert^2 + \left\langle \nabla f(\bw_{k,1}), \eta \sum_{i=1}^n(\frac{1}{\sqrt{\bnu_{k,1}}} -\frac{1}{\sqrt{\bnu_{k,i}}}) \odot g_{k,i}\right\rangle 
    \\
    &+\langle \nabla f(\bw_{k,1}),\eta \sum_{i=1}^n\frac{1}{\sqrt{\bnu_{k,1}}} \odot (\nabla f_{\tau_{k,i}} (\bw_{k,1})-g_{k,i})\rangle. 
\end{align*}
We then tackle the three terms respectively and denote $F_1=-n\eta\left\Vert  \frac{1}{\sqrt[4]{\bnu_{k,1}}} \odot \nabla f (\bw_{k,1})\right\Vert^2$, $F_2=\langle \nabla f(\bw_{k,1}), \eta \sum_{i=1}^n(\frac{1}{\sqrt{\bnu_{k,1}}} -\frac{1}{\sqrt{\bnu_{k,i}}}) \odot g_{k,i}\rangle $, and $F_3=\langle \nabla f(\bw_{k,1}),\eta \sum_{i=1}^n\frac{1}{\sqrt{\bnu_{k,1}}} \odot (\nabla f_{\tau_{k,i}} (\bw_{k,1})-g_{k,i})\rangle$. $F_2$ can be bounded as 
\small
\begin{align*}
   &\vert F_2 \vert 
   \\
   =&\vert\langle \nabla f(\bw_{k,1}), \eta \sum_{i=1}^n(\frac{1}{\sqrt{\bnu_{k,1}}} -\frac{1}{\sqrt{\bnu_{k,i}}}) \odot g_{k,i}\rangle\vert
    = \vert\eta\sum_{l=1}^d \partial_l f(\bw_{k,1})  \sum_{i=1}^n(\frac{1}{\sqrt{\bnu_{k,1,l}}} -\frac{1}{\sqrt{\bnu_{k,i,l}}}) g_{k,i,l}\vert
    \\
    \le & \eta\sum_{l=1}^d \sum_{i=1}^n\vert\partial_l f(\bw_{k,1})  (\frac{1}{\sqrt{\bnu_{k,1,l}}} -\frac{1}{\sqrt{\bnu_{k,i,l}}}) g_{k,i,l}\vert
    \overset{(\star)}{=} \eta\sum_{l=1}^d \sum_{i=1}^n (\frac{1}{\sqrt{\bnu_{k,1,l}}} -\frac{1}{\sqrt{\bnu_{k,i,l}}})\vert\partial_l f(\bw_{k,1})\vert   \vert g_{k,i,l}\vert
    \\
    =&\eta\sum_{l=1}^d \sum_{i=1}^n (\frac{1}{\sqrt{\bnu_{k,1,l}}} -\frac{1}{\sqrt{\bnu_{k,i,l}}})\vert\partial_l f(\bw_{k,1})\vert   \vert \partial_l f_{\tau_{k,i}}(\bw_{k,i})\vert
    = \eta\sum_{l=1}^d \sum_{i=1}^n (\frac{1}{\sqrt{\bnu_{k,1,l}}} -\frac{1}{\sqrt{\bnu_{k,\tilde{\tau}_{k,i},l}}})\vert\partial_l f(\bw_{k,1})\vert  \vert \partial_l f_{i}(\bw_{k,\tilde{\tau}_{k,i}})\vert
    \\
    \le & \eta\sum_{l=1}^d \sum_{i=1}^n (\frac{\vert\partial_l f(\bw_{k-1,1})\vert  \vert \partial_l f_{i}(\bw_{k-1,\tilde{\tau}_{k-1,i}})\vert}{\sqrt{\bnu_{k,1,l}}} -\frac{\vert\partial_l f(\bw_{k,1})\vert  \vert \partial_l f_{i}(\bw_{k,\tilde{\tau}_{k,i}})\vert}{\sqrt{\bnu_{k,\tilde{\tau}_{k,i},l}}})
    \\
    &+\eta\sum_{l=1}^d \sum_{i=1}^n \frac{\vert\partial_l f(\bw_{k,1})\vert  \vert \partial_l f_{i}(\bw_{k,\tilde{\tau}_{k,i}})-\partial_l f_{i}(\bw_{k-1,\tilde{\tau}_{k-1,i}})\vert}{\sqrt{\bnu_{k,1,l}}}
    \\
    &+\eta\sum_{l=1}^d \sum_{i=1}^n \frac{\vert\partial_l f(\bw_{k,1})-\partial_l f(\bw_{k-1,1})\vert  \vert \partial_l f_{i}(\bw_{k-1,\tilde{\tau}_{k-1,i}})\vert}{\sqrt{\bnu_{k,1,l}}},
\end{align*}
\normalsize
where Eq. ($\star$) is due to $\bnu_{k,i,l}$ is non-decreasing with respect to $k$ and $i$. Since $\forall k\ge 2$,
\begin{align*}
    \Vert \nabla  f(\bw_{k,1}) -\nabla  f(\bw_{k-1,1})\Vert^2 \le  L^2 \Vert \bw_{k,1} -\bw_{k-1,1} \Vert^2
    \le nL^2 \sum_{j=1}^n \Vert \bw_{k-1,j+1} -\bw_{k-1,j} \Vert^2, 
\end{align*}
and 
\begin{align*}
    \Vert \nabla  f_i(\bw_{k,\tilde{\tau}_{k,i}}) -\nabla  f_i(\bw_{k-1,\tilde{\tau}_{k-1,i}})\Vert^2 \le&  L^2 \Vert\bw_{k,\tilde{\tau}_{k,i}}-\bw_{k-1,\tilde{\tau}_{k-1,i}}\Vert^2
    \\
    \le &2nL^2 \sum_{j=1}^n \Vert \bw_{k-1,j+1} -\bw_{k-1,j} \Vert^2+2nL^2 \sum_{j=1}^n \Vert \bw_{k,j+1} -\bw_{k,j} \Vert^2,
\end{align*}
 then $\forall k \ge 2$ we can further bound $\vert F_2\vert$ as 
\begin{align*}
& \vert F_2\vert 
\\
    \le & \eta\sum_{l=1}^d \sum_{i=1}^n \left(\frac{\vert\partial_l f(\bw_{k-1,1})\vert  \vert \partial_l f_{i}(\bw_{k-1,\tilde{\tau}_{k-1,i}})\vert}{\sqrt{\bnu_{k,1,l}}} -\frac{\vert\partial_l f(\bw_{k,1})\vert  \vert \partial_l f_{i}(\bw_{k,\tilde{\tau}_{k,i}})\vert}{\sqrt{\bnu_{k,\tilde{\tau}_{k,i},l}}}\right)
    \\
    &+\sum_{i=1}^n \frac{1}{2}(\eta L\left\Vert \frac{1  }{\sqrt[4]{\bnu_{k,1}}} \odot \nabla  f(\bw_{k,1})\right\Vert^2+ \frac{\eta}{L\min_{l} \sqrt{\bnu_{1,0,l}}}\Vert \nabla  f_{i}(\bw_{k,\tilde{\tau}_{k,i}})-\nabla f_{i}(\bw_{k-1,\tilde{\tau}_{k-1,i}})\Vert^2)
    \\
    &+\sum_{i=1}^n  \frac{1}{2}(\eta^2L\left\Vert \frac{1  }{\sqrt{\bnu_{k,1}}} \odot \nabla f_{i}(\bw_{k-1,\tilde{\tau}_{k-1,i}})\right\Vert^2+ \frac{1}{L}\Vert \nabla  f(\bw_{k,1})-\nabla  f(\bw_{k-1,1})\Vert^2) 
    \\
    \le & \eta\sum_{l=1}^d \sum_{i=1}^n \left(\frac{\vert\partial_l f(\bw_{k-1,1})\vert  \vert \partial_l f_{i}(\bw_{k-1,\tilde{\tau}_{k-1,i}})\vert}{\sqrt{\bnu_{k,1,l}}} -\frac{\vert\partial_l f(\bw_{k,1})\vert  \vert \partial_l f_{i}(\bw_{k,\tilde{\tau}_{k,i}})\vert}{\sqrt{\bnu_{k,\tilde{\tau}_{k,i},l}}}\right)
    \\
    &+\sum_{i=1}^n \frac{1}{2}\left(\eta L\left\Vert \frac{1  }{\sqrt{\bnu_{k,1}}} \odot \nabla  f(\bw_{k,1})\right\Vert^2+\frac{2\eta nL}{\min_{l} \sqrt{\bnu_{1,0,l}}} \sum_{j=1}^n \Vert \bw_{k,j+1} -\bw_{k,j} \Vert^2\right)
    \\
    &+\sum_{i=1}^n  \frac{1}{2}\left(\eta^2L\left\Vert \frac{1  }{\sqrt{\bnu_{k,1}}} \odot \nabla f_{i}(\bw_{k-1,\tilde{\tau}_{k-1,i}})\right\Vert^2+\left(1+\frac{2\eta}{\min_{l} \sqrt{\bnu_{1,0,l}}}\right)nL \sum_{j=1}^n \Vert \bw_{k-1,j+1} -\bw_{k-1,j} \Vert^2\right)
    \\
    \overset{(\circ)}{\le }&  \eta\sum_{l=1}^d \sum_{i=1}^n \left(\frac{\vert\partial_l f(\bw_{k-1,1})\vert  \vert \partial_l f_{i}(\bw_{k-1,\tilde{\tau}_{k-1,i}})\vert}{\sqrt{\bnu_{k-1,\tilde{\tau}_{k-1,i},l}}} -\frac{\vert\partial_l f(\bw_{k,1})\vert  \vert \partial_l f_{i}(\bw_{k,\tilde{\tau}_{k,i}})\vert}{\sqrt{\bnu_{k,\tilde{\tau}_{k,i},l}}}\right)
    \\
    &+\sum_{i=1}^n \frac{1}{2}\left(\eta L\left\Vert \frac{1  }{\sqrt{\bnu_{k,1}}} \odot \nabla  f(\bw_{k,1})\right\Vert^2+\frac{2\eta n L}{\min_{l} \sqrt{\bnu_{1,0,l}}} \sum_{j=1}^n \Vert \bw_{k,j+1} -\bw_{k,j} \Vert^2\right)
    \\
    &+\sum_{i=1}^n  \frac{1}{2}\left(\eta^2L\left\Vert \frac{1  }{\sqrt{\bnu_{k,1}}} \odot \nabla f_{i}(\bw_{k-1,\tilde{\tau}_{k-1,i}})\right\Vert^2+\left(1+\frac{2\eta}{\min_{l} \sqrt{\bnu_{1,0,l}}}\right)nL \sum_{j=1}^n \Vert \bw_{k-1,j+1} -\bw_{k-1,j} \Vert^2\right),
\end{align*}
where Inequality $(\circ)$ uses the non-increasing property of $\bnu_{k,i,l}$ with respect to $i$.

When $k=1$, we directly bound $\vert F_2 \vert $ by
\begin{equation*}
     \vert F_2 \vert \le \eta\sum_{l=1}^d \sum_{i=1}^n \left(\frac{\vert\partial_l f(\bw_{k,1})\vert  \vert \partial_l f_{i}(\bw_{k,\tilde{\tau}_{k,i}})\vert}{\sqrt{\bnu_{k,1,l}}} -\frac{\vert\partial_l f(\bw_{k,1})\vert  \vert \partial_l f_{i}(\bw_{k,\tilde{\tau}_{k,i}})\vert}{\sqrt{\bnu_{k,\tilde{\tau}_{k,i},l}}}\right).
\end{equation*}
Meanwhile, $F_3$ can be bounded by
\begin{align*}
    \vert F_3\vert \le &\eta\sum_{l=1}^d\sum_{i=1}^n \frac{1}{\sqrt{\bnu_{k,1,l}}}\vert\partial_l f(\bw_{k,1})\vert   \vert \partial_l f_{\tau_{k,i}} (\bw_{k,1})-g_{k,i,l}\vert 
    \\
    = &\eta\sum_{l=1}^d\sum_{i=1}^n \frac{1}{\sqrt{\bnu_{k,1,l}}}\vert\partial_l f(\bw_{k,1})\vert   \vert \partial_l f_{\tau_{k,i}} (\bw_{k,1})-\partial_l f_{\tau_{k,i}} (\bw_{k,i})\vert
    \\
    \le & \sum_{i=1}^n (\frac{1}{4}\eta L\Vert \frac{1}{\sqrt[4]{\bnu_{k,1}}}\odot\nabla f(\bw_{k,1}) \Vert^2   +\frac{\eta}{L\min_l \sqrt{\bnu_{1,0,l}}}\Vert \nabla f_{\tau_{k,i}} (\bw_{k,1})-\nabla f_{\tau_{k,i}} (\bw_{k,i})\Vert^2)
    \\
    \le & \frac{n}{4}\eta L\Vert \frac{1}{\sqrt{\bnu_{k,1}}}\odot\nabla f(\bw_{k,1}) \Vert^2   +\sum_{i=1}^n \frac{\eta nL}{\min_l \sqrt{\bnu_{1,0,l}}} \sum_{j=1}^n \Vert \bw_{k,j+1}-\bw_{k,j} \Vert^2
    \\
    =& \frac{n}{4}\eta L\left\Vert \frac{1}{\sqrt[4]{\bnu_{k,1}}}\odot\nabla f(\bw_{k,1}) \right\Vert^2   +\frac{\eta}{\min_{l} \sqrt{\bnu_{1,0,l}}} \sum_{j=1}^n n^2L \Vert \bw_{k,j+1}-\bw_{k,j} \Vert^2.
\end{align*}
As a conclusion, when $k\ge 2$, the first order term can be bounded as
\begin{align*}
    &\langle \nabla f(\bw_{k,1}) , \bw_{k+1}-\bw_{k,1} \rangle
    \le  F_1+\vert F_2\vert +\vert F_3\vert
   \\
    \le & -n\eta\left\Vert  \frac{1}{\sqrt[4]{\bnu_{k,1}}} \odot \nabla f (\bw_{k,1})\right\Vert^2+\eta\sum_{l=1}^d \sum_{i=1}^n (\frac{\vert\partial_l f(\bw_{k-1,1})\vert  \vert \partial_l f_{i}(\bw_{k-1,\tilde{\tau}_{k-1,i}})\vert}{\sqrt{\bnu_{k-1,\tilde{\tau}_{k-1,i},l}}} -\frac{\vert\partial_l f(\bw_{k,1})\vert  \vert \partial_l f_{i}(\bw_{k,\tilde{\tau}_{k,i}})\vert}{\sqrt{\bnu_{k,\tilde{\tau}_{k,i},l}}})
    \\
    &+ \frac{3}{4}n\eta L\Vert \frac{1}{\sqrt[4]{\bnu_{k,1}}}\odot\nabla f(\bw_{k,1}) \Vert^2 +\sum_{i=1}^n  \frac{1}{2}\eta^2L\left\Vert \frac{1  }{\sqrt{\bnu_{k,1}}} \odot \nabla f_{i}(\bw_{k-1,\tilde{\tau}_{k-1,i}})\right\Vert^2
    \\
    &+ \left(\frac{1}{2}+\frac{\eta}{\min_{l} \sqrt{\bnu_{1,0,l}}}\right)n^2L\sum_{j=1}^n \Vert \bw_{k-1,j+1}-\bw_{k-1,j} \Vert^2+
\frac{2\eta }{\min_{l} \sqrt{\bnu_{1,0,l}}}n^2L\sum_{j=1}^n  \Vert \bw_{k,j+1}-\bw_{k,j} \Vert^2.
\end{align*}

When $k=1$, the first order term can be bounded as 
\begin{align*}
&\langle \nabla f(\bw_{k,1}) , \bw_{k+1}-\bw_{k,1} \rangle
    \le  F_1+\vert F_2\vert +\vert F_3\vert
    \\
    \le & -n\eta\left\Vert  \frac{1}{\sqrt[4]{\bnu_{k,1}}} \odot \nabla f (\bw_{k,1})\right\Vert^2+\eta\sum_{l=1}^d \sum_{i=1}^n \left(\frac{\vert\partial_l f(\bw_{k,1})\vert  \vert \partial_l f_{i}(\bw_{k,\tilde{\tau}_{k,i}})\vert}{\sqrt{\bnu_{k,1,l}}} -\frac{\vert\partial_l f(\bw_{k,1})\vert  \vert \partial_l f_{i}(\bw_{k,\tilde{\tau}_{k,i}})\vert}{\sqrt{\bnu_{k,\tilde{\tau}_{k,i},l}}}\right)
    \\
    &+\frac{n}{2}\eta L\Vert \frac{1}{\sqrt[4]{\bnu_{k,1}}}\odot\nabla f(\bw_{k,1}) \Vert^2   +\frac{\eta}{\min_l \sqrt{\bnu_{1,0,l}}}\sum_{j=1}^n n^2L \Vert \bw_{k,j+1}-\bw_{k,j} \Vert^2.
\end{align*}

Besides, the second order term can be bounded as
\begin{align*}
    \frac{L}{2} \Vert \bw_{k+1,1}-\bw_{k,1} \Vert^2\le &\frac{nL}{2} \sum_{j=1}^n \Vert \bw_{k,j+1}-\bw_{k,j} \Vert^2.
\end{align*}

Combining the estimations of the first order term and the second order term together into the descent lemma and summing over $k$ then leads to 
\begin{align*}
    &f(\bw_{T+1,1})
    \\
    \le& f(\bw_{1,1}) -\frac{1}{4}\sum_{k=1} ^T n\eta\left\Vert  \frac{1}{\sqrt[4]{\bnu_{k,1}}} \odot \nabla f (\bw_{k,1})\right\Vert^2
    \\
    &+\eta\sum_{l=1}^d \sum_{i=1}^n \left(\frac{\vert\partial_l f(\bw_{1,1})\vert  \vert \partial_l f_{i}(\bw_{1,\tilde{\tau}_{1,i}})\vert}{\sqrt{\bnu_{1,1,l}}} -\frac{\vert\partial_l f(\bw_{1,1})\vert  \vert \partial_l f_{i}(\bw_{1,\tilde{\tau}_{1,i}})\vert}{\sqrt{\bnu_{1,\tilde{\tau}_{1,i},l}}}\right)
    \\
    &+ \eta\sum_{k=2}^T\sum_{l=1}^d \sum_{i=1}^n \left(\frac{\vert\partial_l f(\bw_{k-1,1})\vert  \vert \partial_l f_{i}(\bw_{k-1,\tilde{\tau}_{k-1,i}})\vert}{\sqrt{\bnu_{k-1,\tilde{\tau}_{k-1,i},l}}} -\frac{\vert\partial_l f(\bw_{k,1})\vert  \vert \partial_l f_{i}(\bw_{k,\tilde{\tau}_{k,i}})\vert}{\sqrt{\bnu_{k,\tilde{\tau}_{k,i},l}}}\right)
    \\
    & +\sum_{k=1}^T\sum_{i=1}^n  \frac{1}{2}\eta^2L\left\Vert \frac{1  }{\sqrt{\bnu_{k,1}}} \odot \nabla f_{i}(\bw_{k,\tilde{\tau}_{k,i}})\right\Vert^2
    +\left(1+\frac{3\eta}{\min_{l} \sqrt{\bnu_{1,0,l}}}\right)n^2L\sum_{k=1}^{T} \sum_{j=1}^n \Vert \bw_{k,j+1}-\bw_{k,j} \Vert^2
    \\
    =& f(\bw_{1,1}) -\sum_{k=1} ^T n\eta\left\Vert  \frac{1}{\sqrt[4]{\bnu_{k,1}}} \odot \nabla f (\bw_{k,1})\right\Vert^2
    \\
    &+\eta\sum_{l=1}^d \sum_{i=1}^n \left(\frac{\vert\partial_l f(\bw_{1,1})\vert  \vert \partial_l f_{i}(\bw_{1,\tilde{\tau}_{1,i}})\vert}{\sqrt{\bnu_{1,1,l}}} -\frac{\vert\partial_l f(\bw_{T,1})\vert  \vert \partial_l f_{i}(\bw_{T,\tilde{\tau}_{T,i}})\vert}{\sqrt{\bnu_{T,\tilde{\tau}_{T,i},l}}}\right)
    \\
    &+\sum_{k=1}^T\sum_{i=1}^n  \frac{1}{2}\eta^2L\left\Vert \frac{1  }{\sqrt{\bnu_{k+1,1}}} \odot \nabla f_{i}(\bw_{k,\tilde{\tau}_{k,i}})\right\Vert^2
    +\left(1+\frac{3\eta}{\min_{l} \sqrt{\bnu_{1,0,l}}}\right)n^2L\sum_{k=1}^{T} \sum_{j=1}^n \Vert \bw_{k,j+1}-\bw_{k,j} \Vert^2.
\end{align*}

We respectively bound the last two terms.

The term $\sum_{k=1}^T\sum_{i=1}^n  \frac{1}{2}\eta^2L\left\Vert \frac{1  }{\sqrt{\bnu_{k+1,1}}} \odot \nabla f_{i}(\bw_{k,\tilde{\tau}_{k,i}})\right\Vert^2$ can be bounded as
\begin{align}
\nonumber
    &\sum_{k=1}^T\sum_{i=1}^n  \frac{1}{2}\eta^2L\left\Vert \frac{1  }{\sqrt{\bnu_{k+1,1}}} \odot \nabla f_{i}(\bw_{k,\tilde{\tau}_{k,i}})\right\Vert^2\le  \sum_{k=1}^T\sum_{i=1}^n  \frac{1}{2}\eta^2L\left\Vert \frac{1  }{\sqrt{\bnu_{k,\tilde{\tau}_{k,i}}}} \odot \nabla f_{i}(\bw_{k,\tilde{\tau}_{k,i}})\right\Vert^2
    \\
    \label{eq: estimation_sum}
    =& \sum_{k=1}^T\sum_{i=1}^n  \frac{1}{2}L\left\Vert \bw_{k,j+1}-\bw_{k,j}\right\Vert^2.
\end{align}
Thus, we only need to bound the last term in order to bound the first two terms. The third term $\sum_{k=1}^{T} \left(1+\frac{3\eta}{\min_{l} \sqrt{\bnu_{1,0,l}}}\right)nL^2\sum_{j=1}^n \Vert \bw_{k,j+1}-\bw_{k,j} \Vert^2$ can be bounded as 
\begin{align*}
    \sum_{k=1}^{T} \left(1+\frac{3\eta}{\min_{l} \sqrt{\bnu_{1,0,l}}}\right)nL^2\sum_{j=1}^n \Vert \bw_{k,j+1}-\bw_{k,j} \Vert^2=& \left(1+\frac{3\eta}{\min_{l} \sqrt{\bnu_{1,0,l}}}\right) nL^2\sum_{k=1}^T \sum_{j=1}^n\sum_{l=1}^d \frac{g_{k,i,l}^2}{\bnu_{k,i,l}}
    \\
    \le&\left(1+\frac{3\eta}{\min_{l} \sqrt{\bnu_{1,0,l}}}\right)nL^2 \sum_{l=1}^d \ln  \frac{\bnu_{T,n,l}}{\bnu_{1,0,l}}.
\end{align*}
Combining the above estimations together, we obtain
\begin{align}
\nonumber
     f(\bw_{T+1,1})\le& f(\bw_{1,1}) -\frac{1}{4}\sum_{k=1} ^T 
 n\eta\left\Vert  \frac{1}{\sqrt[4]{\bnu_{k,1}}} \odot \nabla f (\bw_{k,1})\right\Vert^2
    +\eta\sum_{l=1}^d \sum_{i=1}^n \frac{\vert\partial_l f(\bw_{1,1})\vert  \vert \partial_l f_{i}(\bw_{1,\tilde{\tau}_{1,i}})\vert}{\sqrt{\bnu_{1,1,l}}} 
    \\
    \label{eq: adagrad_rr_first_stage}
    &+\left(\left(1+\frac{3\eta}{\min_{l} \sqrt{\bnu_{1,0,l}}}\right)nL^2+\frac{1}{2}L \right)
    \sum_{l=1}^d(\ln \bnu_{T,n,l}-\ln \bnu_{1,0,l}).
\end{align}

Define
\small
\begin{gather*}
    C_0 \triangleq f(\bw_{1,1}) -f(\bw_{T+1,1}) +\eta\sum_{l=1}^d \sum_{i=1}^n \frac{\vert\partial_l f(\bw_{1,1})\vert  \vert \partial_l f_{i}(\bw_{1,\tilde{\tau}_{1,i}})\vert}{\sqrt{\bnu_{1,1,l}}}+\left(\left(1+\frac{3\eta}{\min_{l} \sqrt{\bnu_{1,0,l}}}\right)nL^2+\frac{1}{2}L \right)
    \sum_{l=1}^d\ln \bnu_{1,0,l},
    \\
    C_1\triangleq  \left(1+\frac{3\eta}{\min_{l} \sqrt{\bnu_{1,0,l}}}\right)nL^2+\frac{1}{2}L  .
\end{gather*}
\normalsize
Based on these notations, Eq. (\ref{eq: adagrad_rr_first_stage}) can be abbreviated as
\begin{align}
\label{eq: adagrad_rr_first_stage_abb}
    \sum_{k=1} ^T n\eta\left\Vert  \frac{1}{\sqrt[4]{\bnu_{k,1}}} \odot \nabla f (\bw_{k,1})\right\Vert^2 \le 4C_0+4C_1 \sum_{l=1}^d \ln\bnu_{T,n,l}.
\end{align}
Since 
\begin{align*}
    \left\Vert  \frac{1}{\sqrt[4]{\bnu_{k,1}}} \odot \nabla f (\bw_{k,1})\right\Vert^2= \sum_{l=1}^d \frac{\partial_l f(\bw_{k,1})^2}{\sqrt{\bnu_{k,1,l}}} \ge \sum_{l=1}^d \frac{\partial_l f(\bw_{k,1})^2}{\sqrt{\sum_{l'=1}^d\bnu_{k,1,l'}}}= \frac{\Vert \nabla f(\bw_{k,1}) \Vert^2}{\sqrt{\sum_{l'=1}^d\bnu_{k,1,l'}}},
\end{align*}
the LHS of Eq. (\ref{eq: adagrad_rr_first_stage_abb}) can be lower bounded as
\begin{align*}
      &\sum_{k=1} ^T n\eta\left\Vert  \frac{1}{\sqrt[4]{\bnu_{k,1}}} \odot \nabla f (\bw_{k,1})\right\Vert^2 \ge \sum_{k=1} ^T n\eta \frac{\Vert \nabla f(\bw_{k,1}) \Vert^2}{\sqrt{\sum_{l'=1}^d\bnu_{k,1,l'}}} \ge n\eta \frac{ \sum_{k=1} ^T\Vert \nabla f(\bw_{k,1}) \Vert^2}{\sqrt{\sum_{l'=1}^d\bnu_{T,1,l'}}}
      \\
      \overset{(\star)}{\ge} & n\eta \frac{ \frac{1}{n D_1} \sum_{k=1} ^T\sum_{i=1}^n \Vert \nabla f_i(\bw_{k,1}) \Vert^2-\frac{D_0}{D_1}T}{\sqrt{\sum_{l'=1}^d\bnu_{T,1,l'}}} 
      \\
      \ge & n\eta \frac{ \frac{1}{2nD_1} \sum_{k=1} ^T\sum_{i=1}^n \Vert \nabla f_i(\bw_{k,\tilde{\tau}_{k,i}}) \Vert^2- \frac{1}{nD_1} \sum_{k=1} ^T\sum_{i=1}^n \Vert \nabla f_i(\bw_{k,\tilde{\tau}_{k,i}})- \nabla f_i(\bw_{k,1}) \Vert^2-\frac{D_0}{D_1}T}{\sqrt{\sum_{l'=1}^d\bnu_{T,1,l'}}}
      \\
      \ge & n\eta \frac{ \frac{1}{2nD_1} \sum_{k=1} ^T\sum_{i=1}^n \Vert \nabla f_i(\bw_{k,\tilde{\tau}_{k,i}}) \Vert^2- \frac{L^2}{nD_1} \sum_{k=1} ^T\sum_{i=1}^n \Vert \bw_{k,\tilde{\tau}_{k,i}}- \bw_{k,1} \Vert^2-\frac{D_0}{D_1}T}{\sqrt{\sum_{l'=1}^d\bnu_{T,1,l'}}}
       \end{align*}
       \begin{align*}
      \ge & n\eta \frac{ \frac{1}{2nD_1} \sum_{k=1} ^T\sum_{i=1}^n \Vert \nabla f_i(\bw_{k,\tilde{\tau}_{k,i}}) \Vert^2- \frac{L^2}{D_1} \sum_{k=1} ^T\sum_{i=1}^n \sum_{j=1}^n \Vert \bw_{k,j+1}- \bw_{k,j} \Vert^2-\frac{D_0}{D_1}T}{\sqrt{\sum_{l'=1}^d\bnu_{T,1,l'}}}
      \\
      =& n\eta \frac{ \frac{1}{2nD_1} \sum_{k=1} ^T\sum_{i=1}^n \Vert \nabla f_i(\bw_{k,\tilde{\tau}_{k,i}}) \Vert^2- \frac{n
      L^2}{D_1} \sum_{k=1} ^T \sum_{j=1}^n \Vert \bw_{k,j+1}- \bw_{k,j} \Vert^2-\frac{D_0}{D_1}T}{\sqrt{\sum_{l'=1}^d\bnu_{T,1,l'}}}
      \\
      \ge & n\eta \frac{ \frac{1}{2nD_1}(\sum_{l'=1}^d\bnu_{T,1,l'}-\sum_{l'=1}^d\bnu_{1,0,l'})- \frac{n      L^2}{D_1} \sum_{k=1} ^T \sum_{j=1}^n \Vert \bw_{k,j+1}- \bw_{k,j} \Vert^2-\frac{D_0}{D_1}T}{\sqrt{\sum_{l'=1}^d\bnu_{T,1,l'}}}.
\end{align*}

As a conclusion, 
\begin{align*}
     &n\eta \frac{ \frac{1}{2nD_1}(\sum_{l'=1}^d\bnu_{T,n,l'}-\sum_{l'=1}^d\bnu_{1,0,l'})- \frac{n
      L^2}{D_1} \sum_{k=1} ^T \sum_{j=1}^n \Vert \bw_{k,j+1}- \bw_{k,j} \Vert^2-\frac{D_0}{D_1}T}{\sqrt{\sum_{l'=1}^d\bnu_{T,1,l'}}}
      \\
      \le &4C_0 +4C_1 \sum_{l=1}^d \ln\bnu_{T,n,l}.
\end{align*}

Echoing Eq. (\ref{eq: estimation_sum}), we then obtain 
\begin{align*}
     &n\eta \frac{ \frac{1}{2nD_1}(\sum_{l'=1}^d\bnu_{T,n,l'}-\sum_{l'=1}^d\bnu_{1,0,l'})- \frac{n
      L^2}{D_1} \eta^2\sum_{l=1}^d \ln \frac{\bnu_{T,n,l}}{\bnu_{1,0,l}}-\frac{D_0}{D_1}T}{\sqrt{\sum_{l'=1}^d\bnu_{T,1,l'}}}
      \\
      \le &4C_0 +4C_1 \sum_{l=1}^d \ln\bnu_{T,n,l}.
\end{align*}

Solving the above inequality with respect to $\sum_{l'=1}^d \bnu_{T,n,l'}$ leads to
\begin{equation*}
    \sum_{l'=1}^d \bnu_{T,n,l'}=\mathcal{O}(1)+\mathcal{O}(D_0T).
\end{equation*}

Applying the above inequality back to Eq. (\ref{eq: adagrad_rr_first_stage_abb}) completes the proof.
\end{proof}

\section{Analysis under $(L_0,L_1)$-smooth condition}
\subsection{Proof of Theorem \ref{thm: adagrad_norm_nonsmooth}}
\label{appen: adagrad_ns}
Before the formal proof of Theorem \ref{thm: adagrad_norm_nonsmooth}, we first introduce the version of descent lemma under $(L_0,L_1)$-smooth condition.
\begin{lemma}[Descent lemma under $(L_0,L_1)$-smooth condition]
\label{lem: descent_nonsmooth}
Let Assumption \ref{assum: non-smooth} holds. Then, if $\bw_1,\bw_2\in \mathbb{R}^d$ satisfies $\Vert \bw_1-\bw_2 \Vert \le \frac{1}{L_1}$, then
\begin{align*}
    f(\bw_1)\le f(\bw_2)+\langle \nabla f(\bw_2), \bw_1-\bw_2 \rangle +\frac{(L_0+L_1 \Vert \nabla f(\bw_2)\Vert)}{2} \Vert \bw_1 -\bw_2 \Vert^2.
\end{align*}    
\end{lemma}

The proof bears great similarity to that of the descent lemma under $L$-smooth condition and we omit it here. Interested readers can refer to \citep[Lemma A.3]{zhang2020improved} for details.

\begin{proof}[Proof of Theorem \ref{thm: adagrad_norm_nonsmooth}]
    Since we set $\eta\le \frac{1}{L_1}$ and $\Vert \frac{\nabla f(\bw_t)}{\sqrt{\bnu_t}} \Vert \le 1$, we have
    \begin{equation*}
        \Vert \bw_{t+1}-\bw_{t} \Vert =\eta \left\Vert \frac{\nabla f(\bw_t)}{\sqrt{\bnu_t}} \right\Vert\le \frac{1}{L_1}.
    \end{equation*}
Therefore, Lemma \ref{lem: descent_nonsmooth} can be applied to $\bw_{t}$ and $\bw_{t+1}$, taking expectation to which leads to
\begin{align}
\nonumber
    \mathbb{E}_{\mathcal{F}_t}f(\bw_{t+1})\le  f(\bw_{t}) + \begin{matrix}\underbrace{\mathbb{E}_{\mathcal{F}_t}\langle \nabla f(\bw_t),-\eta\frac{ \nabla f(\bw_t) }{\sqrt{\bnu_{t}} } \rangle}
   \\
    \text{First Order}
    \end{matrix} + \begin{matrix}\underbrace{\frac{L_0+L_1\Vert \nabla f(\bw_t)\Vert}{2}\eta^2\mathbb{E}_{\mathcal{F}_t} \left\Vert  \frac{g_t}{\sqrt{\bnu_t} } \right\Vert^2}
   \\
    \text{Second Order}
   \end{matrix} .
\end{align}

We then decompose the "First Order" term into the "First Order Main" term and the "Error" term as follows.
\begin{align}
\nonumber
    \mathbb{E}_{\mathcal{F}_t}f(\bw_{t+1})\le&  f(\bw_{t}) + \begin{matrix}\underbrace{-\eta\frac{\Vert \nabla f(\bw_t)\Vert^2 }{\sqrt{\bnu_{t-1}} }}
   \\
    \text{First Order Main}
    \end{matrix} + \begin{matrix}\underbrace{\mathbb{E}_{\mathcal{F}_t}\left[\left\langle \nabla f(\bw_t), \eta (\frac{1}{\sqrt{\bnu_{t-1}} }-\frac{1}{\sqrt{\bnu_{t}} })g_t \right\rangle\right] }
   \\
    \text{Error}
    \end{matrix}
    \\
    \nonumber
    +& \begin{matrix}\underbrace{\frac{L_0+L_1\Vert \nabla f(\bw_t)\Vert}{2}\eta^2\mathbb{E}_{\mathcal{F}_t} \left\Vert  \frac{g_t}{\sqrt{\bnu_t} } \right\Vert^2}
   \\
    \text{Second Order}
   \end{matrix}
   \\
   \nonumber
   \le&  f(\bw_{t}) + \begin{matrix}\underbrace{-\eta\frac{\Vert \nabla f(\bw_t)\Vert^2 }{\sqrt{\bnu_{t-1}} }}
   \\
    \text{First Order Main}
    \end{matrix} + \begin{matrix}\underbrace{\mathbb{E}_{\mathcal{F}_t}\left[\left\langle \nabla f(\bw_t), \eta (\frac{1}{\sqrt{\bnu_{t-1}} }-\frac{1}{\sqrt{\bnu_{t}} })g_t \right\rangle\right] }
   \\
    \text{Error}
    \end{matrix}
    \\
    \label{eq: descent_non_smooth_proof}
    +& \begin{matrix}\underbrace{\frac{L_0}{2}\eta^2\mathbb{E}_{\mathcal{F}_t} \left\Vert  \frac{g_t}{\sqrt{\bnu_t} } \right\Vert^2}
   \\
    \text{Second Order I}
   \end{matrix}+\begin{matrix}\underbrace{\frac{L_1\Vert \nabla f(\bw_t)\Vert}{2}\eta^2\mathbb{E}_{\mathcal{F}_t} \left\Vert  \frac{g_t}{\sqrt{\bnu_t} } \right\Vert^2}
   \\
    \text{Second Order II}
   \end{matrix}.
\end{align}
Here the second inequality is obtained by directly expanding the "Second Order" term. The  "Second Order I" term takes the same form of the "Second Order" term in the proof of Theorem \ref{thm: adagrad_norm}, and it can be handled in the same way as the proof of Theorem \ref{thm: adagrad_norm}. So does the "First Order Main" term. As for the "Error" term, following the same routine as the proof of Lemma \ref{lem: auxiliary}, we obtain that
 \begin{align*}
\left\vert\mathbb{E}_{\mathcal{F}_t}\left[\left\langle \nabla f(\bw_t), \eta (\frac{1}{\sqrt{\bnu_{t-1}} }-\frac{1}{\sqrt{\bnu_{t}} })g_t \right\rangle\right]\right \vert\le  \eta\frac{\Vert \nabla f(\bw_t)\Vert}{\sqrt{\bnu_{t-1}} }\mathbb{E}_{\mathcal{F}_t}\left[  \frac{\Vert g_t\Vert^2}{\sqrt{\bnu_{t}}+\sqrt{\bnu_{t-1}}}\right].
 \end{align*}
Meanwhile, since $\bnu_t$ is non-decreasing with respect to $t$, we have that
\begin{equation*}
    \frac{L_1\Vert \nabla f(\bw_t)\Vert}{2}\eta^2\mathbb{E}_{\mathcal{F}_t} \left\Vert  \frac{g_t}{\sqrt{\bnu_t} } \right\Vert^2 \le L_1\eta^2\frac{\Vert \nabla f(\bw_t)\Vert}{\sqrt{\bnu_{t-1}} }\mathbb{E}_{\mathcal{F}_t}\left[  \frac{\Vert g_t\Vert^2}{\sqrt{\bnu_{t}}+\sqrt{\bnu_{t-1}}}\right].
\end{equation*}
One can easily observe that the RHSs of the above two inequalities have the same form except the coefficients! Thus, the "Error" term plus the "Second Order II" can be bounded by
\begin{align*}
&\begin{matrix}\underbrace{\mathbb{E}_{\mathcal{F}_t}\left[\left\langle \nabla f(\bw_t), \eta (\frac{1}{\sqrt{\bnu_{t-1}} }-\frac{1}{\sqrt{\bnu_{t}} })g_t \right\rangle\right] }
   \\
    \text{Error}
    \end{matrix}+\begin{matrix}\underbrace{\frac{L_1\Vert \nabla f(\bw_t)\Vert}{2}\eta^2\mathbb{E}_{\mathcal{F}_t} \left\Vert  \frac{g_t}{\sqrt{\bnu_t} } \right\Vert^2}
   \\
    \text{Second Order II}
   \end{matrix}
   \\
   \le& (\eta+L_1\eta^2)\frac{\Vert \nabla f(\bw_t)\Vert}{\sqrt{\bnu_{t-1}} }\mathbb{E}_{\mathcal{F}_t}\left[  \frac{\Vert g_t\Vert^2}{\sqrt{\bnu_{t}}+\sqrt{\bnu_{t-1}}}\right] \le 2\eta\frac{\Vert \nabla f(\bw_t)\Vert}{\sqrt{\bnu_{t-1}} }\mathbb{E}_{\mathcal{F}_t}\left[  \frac{\Vert g_t\Vert^2}{\sqrt{\bnu_{t}}+\sqrt{\bnu_{t-1}}}\right],
\end{align*}
where in the last inequality we use $\eta L_1 \le 1$. By the mean-value inequality ($2ab\le a^2+b^2$), 
\begin{align*}
    2\eta\frac{\Vert \nabla f(\bw_t)\Vert}{\sqrt{\bnu_{t-1}} }\mathbb{E}_{\mathcal{F}_t}\left[  \frac{\Vert g_t\Vert^2}{\sqrt{\bnu_{t}}+\sqrt{\bnu_{t-1}}}\right] \le& \frac{1}{2}\eta\frac{\Vert \nabla f(\bw_t)\Vert^2}{\sqrt{\bnu_{t-1}} }+2\frac{\eta}{\sqrt{\bnu_{t-1}} }\left(\mathbb{E}_{\mathcal{F}_t}\left[  \frac{\Vert g_t\Vert^2}{\sqrt{\bnu_{t}}+\sqrt{\bnu_{t-1}}}\right]\right)^2
    \\
    \le & \frac{1}{2}\eta\frac{\Vert \nabla f(\bw_t)\Vert^2}{\sqrt{\bnu_{t-1}} }+2\frac{\eta}{\sqrt{\bnu_{t-1}} }\mathbb{E}_{\mathcal{F}_t}  \Vert g_t \Vert^2 \cdot\mathbb{E}_{\mathcal{F}_t}\left[  \frac{\Vert g_t\Vert^2}{(\sqrt{\bnu_{t}}+\sqrt{\bnu_{t-1}})^2}\right]
    \\
    \le & \frac{1}{2}\eta\frac{\Vert \nabla f(\bw_t)\Vert^2}{\sqrt{\bnu_{t-1}} }+2 \frac{\eta}{\sqrt{\bnu_{t-1}} }( D_0+ D_1 \Vert \nabla f(\bw_{t}) \Vert^2)  \mathbb{E}_{\mathcal{F}_t}\left[  \frac{\Vert g_t\Vert^2}{(\sqrt{\bnu_{t}}+\sqrt{\bnu_{t-1}})^2}\right],
\end{align*}
where the second inequality is due to Hölder's inequality, and in the last inequality we use Assumption \ref{assum: smooth}. Similar to the proof of Theorem \ref{thm: adagrad_norm}, we focus on the term $2D_1 \Vert \nabla f(\bw_t) \Vert^2 \mathbb{E}_{\mathcal{F}_t}\left[  \frac{\Vert g_t\Vert^2}{(\sqrt{\bnu_{t}}+\sqrt{\bnu_{t-1}})^2}\right]$, since the rest of the terms can be easily bounded. Such a term can be bounded as
\begin{align*}
    &2\frac{\eta D_1 \Vert \nabla f(\bw_t) \Vert^2}{\sqrt{\bnu_{t-1}} } \mathbb{E}_{\mathcal{F}_t}\left[  \frac{\Vert g_t\Vert^2}{(\sqrt{\bnu_{t}}+\sqrt{\bnu_{t-1}})^2}\right]
   \\
   \le& 2\eta D_1 \Vert \nabla f(\bw_t) \Vert^2 \mathbb{E}_{\mathcal{F}_t}\left(\frac{1}{\sqrt{\bnu_{t-1}}}-\frac{1}{\sqrt{\bnu_{t}}}\right)
    \\
    =&  2\eta  D_1  \mathbb{E}_{\mathcal{F}_t}\left(\frac{\Vert \nabla f(\bw_{t-1}) \Vert^2}{\sqrt{\bnu_{t-1}}}-\frac{\Vert \nabla f(\bw_t) \Vert^2}{\sqrt{\bnu_{t}}}\right)
    +2 \eta D_1  \frac{\Vert \nabla f(\bw_{t}) \Vert^2-\Vert \nabla f(\bw_{t-1}) \Vert^2}{\sqrt{\bnu_{t-1}}}.
\end{align*}
By  Assumption \ref{assum: non-smooth}, $\Vert \nabla f(\bw_{t}) \Vert-\Vert \nabla f(\bw_{t-1}) \Vert \le \Vert  \nabla f(\bw_{t})- \nabla f(\bw_{t-1}) \Vert \le (L_0 +L_1 \Vert \nabla f(\bw_t) \Vert )\Vert \bw_t-\bw_{t-1} \Vert  $. Therefore, 
\begin{align*}
    &2\eta D_1 \Vert \nabla f(\bw_t) \Vert^2 \mathbb{E}_{\mathcal{F}_t}\left(\frac{1}{\sqrt{\bnu_{t-1}}}-\frac{1}{\sqrt{\bnu_{t}}}\right)
    \\
    =&  
   2\eta  D_1  \frac{2(L_0+L_1\Vert \nabla f(\bw_t) \Vert )\Vert \bw_{t}-\bw_{t-1}\Vert \Vert \nabla f(\bw_t)\Vert +(L_0+L_1\Vert \nabla f(\bw_t) \Vert )^2 \Vert \bw_{t}-\bw_{t-1}\Vert^2}{\sqrt{\bnu_{t-1}}}
    \\
    &+2 \eta  D_1  \mathbb{E}_{\mathcal{F}_t}\left(\frac{\Vert \nabla f(\bw_{t-1}) \Vert^2}{\sqrt{\bnu_{t-1}}}-\frac{\Vert \nabla f(\bw_t) \Vert^2}{\sqrt{\bnu_{t}}}\right)
    \\
    \le & 2\eta  D_1  \frac{2(L_0+L_1\Vert \nabla f(\bw_t) \Vert )\Vert \bw_{t}-\bw_{t-1}\Vert \Vert \nabla f(\bw_t)\Vert +2(L_0^2+L_1^2\Vert \nabla f(\bw_t) \Vert^2 ) \Vert \bw_{t}-\bw_{t-1}\Vert^2}{\sqrt{\bnu_{t-1}}}
    \\
    &+2 \eta  D_1  \mathbb{E}_{\mathcal{F}_t}\left(\xi(t-1)-\xi(t)\right).
\end{align*}
where the last inequality we use $\xi(t)\triangleq \frac{\Vert \nabla f(\bw_t)\Vert^2}{\sqrt{\bnu_t}}$ and the mean-value inequality $(L_0+L_1\Vert \nabla f(\bw_t) \Vert)^2 \le 2(L_0^2+L_1^2 \Vert \nabla f(\bw_t)\Vert^2)$. Reorganize the RHS of the above inequality then leads to
\begin{align*}
     &2\eta  D_1  \frac{2(L_0+L_1\Vert \nabla f(\bw_t) \Vert )\Vert \bw_{t}-\bw_{t-1}\Vert \Vert \nabla f(\bw_t)\Vert +2(L_0^2+L_1^2\Vert \nabla f(\bw_t) \Vert^2 ) \Vert \bw_{t}-\bw_{t-1}\Vert^2}{\sqrt{\bnu_{t-1}}}
    \\
    &+2 \eta  D_1  \mathbb{E}_{\mathcal{F}_t}\left(\xi(t-1)-\xi(t)\right)
    \\
    =&4\eta D_1L_0\frac{\Vert w_t-w_{t-1} \Vert \Vert \nabla f(\bw_t) \Vert }{\sqrt{\bnu_{t-1}}}+4\eta D_1L_1\frac{\Vert w_t-w_{t-1} \Vert \Vert \nabla f(\bw_t) \Vert ^2 }{\sqrt{\bnu_{t-1}}}+4L_0^2\eta 
 D_1\frac{\Vert w_t-w_{t-1} \Vert^2 }{\sqrt{\bnu_{t-1}}}
    \\
    &+4L_1^2\eta 
 D_1\frac{\Vert w_t-w_{t-1} \Vert^2 \Vert \nabla f(\bw_t) \Vert ^2 }{\sqrt{\bnu_{t-1}}}+2 \eta  D_1  \mathbb{E}_{\mathcal{F}_t}\left(\xi(t-1)-\xi(t)\right).
\end{align*}

In the RHS of the above inequality, the first term can be bounded by the mean-value inequality as
\begin{equation*}
    4\eta D_1L_0\frac{\Vert w_t-w_{t-1} \Vert \Vert \nabla f(\bw_t) \Vert }{\sqrt{\bnu_{t-1}}}\le \frac{\eta}{4} \frac{ \Vert \nabla f(\bw_t) \Vert^2 }{\sqrt{\bnu_{t-1}}}+16 D_1^2L_0^2  \frac{ \Vert w_t-w_{t-1} \Vert^2 }{\sqrt{\bnu_{t-1}}}.
\end{equation*}

By $\eta \le \frac{1}{64D_1L_1}$, the second term can be bounded as 
\begin{equation*}
    4\eta D_1L_1\frac{\Vert w_t-w_{t-1} \Vert \Vert \nabla f(\bw_t) \Vert ^2 }{\sqrt{\bnu_{t-1}}}\le \frac{1}{16} \eta \frac{ \Vert \nabla f(\bw_t) \Vert ^2 }{\sqrt{\bnu_{t-1}}}.
\end{equation*}

By $\eta \le \frac{1}{8\sqrt{D_1}L_1}$, the forth term can be bounded as 
\begin{equation*}
    4\eta D_1L_1^2\frac{\Vert w_t-w_{t-1} \Vert^2 \Vert \nabla f(\bw_t) \Vert ^2 }{\sqrt{\bnu_{t-1}}}\le \frac{1}{16} \eta \frac{ \Vert \nabla f(\bw_t) \Vert ^2 }{\sqrt{\bnu_{t-1}}}.
\end{equation*}

Therefore, $2\eta D_1 \Vert \nabla f(\bw_t) \Vert^2 \mathbb{E}_{\mathcal{F}_t}\left(\frac{1}{\sqrt{\bnu_{t-1}}}-\frac{1}{\sqrt{\bnu_{t}}}\right) $ can be bounded as
\begin{align*}
    &2\eta D_1 \Vert \nabla f(\bw_t) \Vert^2 \mathbb{E}_{\mathcal{F}_t}\left(\frac{1}{\sqrt{\bnu_{t-1}}}-\frac{1}{\sqrt{\bnu_{t}}}\right)
    \\
    \le & \frac{3}{8} \eta \frac{ \Vert \nabla f(\bw_t) \Vert ^2 }{\sqrt{\bnu_{t-1}}}+4L_0^2\eta 
 D_1\frac{\Vert w_t-w_{t-1} \Vert^2 }{\sqrt{\bnu_{t-1}}}+16 D_1^2L_0^2  \frac{ \Vert w_t-w_{t-1} \Vert^2 }{\sqrt{\bnu_{t-1}}}
    \\
    &+2 \eta  D_1  \mathbb{E}_{\mathcal{F}_t}\left(\xi(t-1)-\xi(t)\right)
\end{align*}

All in all, we conclude that the "Error" term plus the "Second Order II" term can be bounded as 
\begin{align*}
&\begin{matrix}\underbrace{\mathbb{E}_{\mathcal{F}_t}\left[\left\langle \nabla f(\bw_t), \eta (\frac{1}{\sqrt{\bnu_{t-1}} }-\frac{1}{\sqrt{\bnu_{t}} })g_t \right\rangle\right] }
   \\
    \text{Error}
    \end{matrix}+\begin{matrix}\underbrace{\frac{L_1\Vert \nabla f(\bw_t)\Vert}{2}\eta^2\mathbb{E}_{\mathcal{F}_t} \left\Vert  \frac{g_t}{\sqrt{\bnu_t} } \right\Vert^2}
   \\
    \text{Second Order II}
   \end{matrix}
   \\
   \le &\frac{7}{8} \eta \frac{ \Vert \nabla f(\bw_t) \Vert ^2 }{\sqrt{\bnu_{t-1}}}+4L_0^2\eta 
 D_1\frac{\Vert w_t-w_{t-1} \Vert^2 }{\sqrt{\bnu_{t-1}}}+16 D_1^2L_0^2  \frac{ \Vert w_t-w_{t-1} \Vert^2 }{\sqrt{\bnu_{t-1}}}
    \\
    &+2 \eta  D_1  \mathbb{E}_{\mathcal{F}_t}\left(\xi(t-1)-\xi(t)\right).
\end{align*}

Applying the above inequality back to Eq. (\ref{eq: descent_non_smooth_proof}), we conclude that
\begin{align*}
    \mathbb{E}_{\mathcal{F}_t}f(\bw_{t+1})\le&  f(\bw_{t})  -\frac{\eta}{8}\frac{\Vert \nabla f(\bw_t)\Vert^2 }{\sqrt{\bnu_{t-1}} }+\frac{L_0}{2}\eta^2\mathbb{E}_{\mathcal{F}_t} \left\Vert  \frac{g_t}{\sqrt{\bnu_t} } \right\Vert^2+4L_0^2\eta
 D_1\frac{\Vert w_t-w_{t-1} \Vert^2 }{\sqrt{\bnu_{t-1}}}
  \\
    &+16 D_1^2L_0^2  \frac{ \Vert w_t-w_{t-1} \Vert^2 }{\sqrt{\bnu_{t-1}}}
   +2 \eta  D_1  \mathbb{E}_{\mathcal{F}_t}\left(\xi(t-1)-\xi(t)\right).
\end{align*}

The above inequality takes the same form of Eq. (\ref{eq: potantial_mid_2}) except the coefficients. The rest of the proof follow the same routine as the rest of the proof of Theorem \ref{thm: adagrad_norm}.

The proof is completed.
\end{proof}

\subsection{Proof of Theorem \ref{thm: counter}}
\label{appen: counter}
Before the formal proof, we first present some notations.
\\
\textbf{Constants.} $\forall t \in \mathbb{Z}^{+}$, we define $a_t=\eta 
\frac{2^t}{1+\cdots+4^t}=\eta \frac{2^t}{\sqrt{\frac{4^{t+1}-1}{3}}}$. One can easily observe that $a_t$ is increasing with respect to $t$ and thus $a_t\ge a_1=\eta \frac{2}{\sqrt{5}}>\frac{9}{L_1}$, where the last inequality we use $\eta> \frac{4\sqrt{5}}{L_1}$. We then define $S_{2k-1}\triangleq \sum_{t=1}^{k} a_{2t-1}$ and $S_{2k}\triangleq \sum_{t=1}^{k} a_{2t}$ with $S_0=0$ and $S_{-1}=0$.
\\
\textbf{Line segments.} We define the following line segments: $\mathcal{C}_{2k-1}=\{(x,S_{2k-2}): x\in [S_{2k-3},S_{2k-1})\}$, and $\mathcal{C}_{2k}=\{(S_{2k-1},y): y\in [S_{2k-2},S_{2k})\}$. Define $\mathcal{C}=\cup_{k=1}^{\infty} \mathcal{C}_{k}$.

\begin{proof}[Proof of Theorem \ref{thm: counter}] We first define a function over $\mathcal{C}$ as follows. Define $f((0,0))=\frac{2}{L_1}$, and 

\begin{equation*}
    \nabla f|_{\mathcal{C}_{2k-1}} (x,S_{2k-2})=\left \{\begin{aligned}
    (2^{2k-1},0)+L_1(x-S_{2k-3}) (-2^{2k-2},2^{2k-2}), ~& x\in [S_{2k-3},S_{2k-3}+\frac{4}{L_1})
    \\
    (-2^{2k-1},2^{2k})+(x-S_{2k-3}-\frac{4}{L_1})(-2^{2k-2},0),~& x\in [S_{2k-3}+\frac{4}{L_1},S_{2k-3}+\frac{8}{L_1})
        \\
    (0,2^{2k}),~& x\in [S_{2k-3}+\frac{8}{L_1},S_{2k-1})
    \end{aligned} \right..
\end{equation*}

\begin{equation*}
    \nabla f|_{\mathcal{C}_{2k}} (S_{2k},y)=\left \{\begin{aligned}
    (0, 2^{2k})+L_1(y-S_{2k-2}) (2^{2k-1},-2^{2k-1}), ~& y\in [S_{2k-2},S_{2k-2}+\frac{4}{L_1})
    \\
    (2^{2k+1},-2^{2k})+(y-S_{2k-2}-\frac{4}{L_1})(0,-2^{2k-1}),~& y\in [S_{2k-2}+\frac{4}{L_1},S_{2k-2}+\frac{8}{L_1})
    \\
    (2^{2k+1},0),~& y\in [S_{2k-2}+\frac{8}{L_1},S_{2k})
    \end{aligned} \right..
\end{equation*}

To begin with, one can easily observe that $\nabla f((S_{2k-1},S_{2k-2}))= (0,2^{2k})$ and $\nabla f((S_{2k-1},S_{2k}))= (2^{2k+1},0)$. We then prove that $f|_{\mathcal{C}}$ obeys $(0,L_1)$-smooth condition: as the length of $\mathcal{C}_k$ is $a_k$, which is longer than $\frac{8}{L_1}$, if $\bw_1$ and $\bw_2$ belong to $\mathcal{C}$ and satisfy $\Vert \bw_1-\bw_2 \Vert \le \frac{1}{L_1}$, it must be that either there exists a $k$, such that $\bw_1,\bw_2 \in \mathcal{C}_k$, or there exists a $k$ such that $\bw_1\in \mathcal{C}_k$ and $\bw_2 \in \mathcal{C}_{k+1}$ (or there exists a $k$ such that $\bw_1\in \mathcal{C}_{k+1}$ and $\bw_2 \in \mathcal{C}_{k}$, which can be tackled in the same way due to the symmetry between $\bw_1$ and $\bw_2$). For the former case, without the loss of generality we assume that $k$ is odd and equals $2k'-1$. We then observe that the gradient norm is no smaller than $2^{2k'-2}\sqrt{2}$ when $\bw \in [S_{2k'-3},S_{2k'-3}+\frac{4}{L_1}]\times \{S_{2k'-2}\}$, while the norm of directional Hessian is $L_12^{2k'-2}\sqrt{2}$. Similarly, we observe that the gradient norm is no smaller than $2^{2k'}$ when $\bw \in [S_{2k'-3}+\frac{4}{L_1},S_{2k'-1}]\times \{S_{2k'-2}\}$, while the norm of directional Hessian is no larger than $L_12^{2k'-2}$. Therefore, we conclude that $\Vert \nabla f(\bw_1) -\nabla f(\bw_2) \Vert \le L_1 \Vert\bw_1-\bw_2 \Vert \Vert \nabla f(\bw_1) \Vert  $. As for the latter case, without the loss of generality we still assume $k=2k'-1$, and $\bw_1\in \mathcal{C}_{2k'-1}$ and $\bw_2 \in \mathcal{C}_{2k'}$. Since $\Vert \bw_1-\bw_2 \Vert \le \frac{1}{L_1}$, we must have $\bw_1 \in [S_{2k'-3}+\frac{8}{L_1}, S_{2k'-1 }]\times \{S_{2k'-2}\}$. Thus,
\begin{align*}
    &\Vert \nabla f(\bw_1)-\nabla f(\bw_2) \Vert
    \\
    =& \Vert \nabla f((S_{2k'-1}, S_{2k'}))-\nabla f(\bw_2) \Vert
    \\
    \le&L_1 \Vert \nabla f(\bw_2) \Vert \Vert (S_{2k'-1}, S_{2k'}) -\bw_2 \Vert \le   L_1 \min\{\Vert \nabla f(\bw_2) \Vert, \Vert \nabla f(\bw_1) \Vert\} \Vert \bw_1 -\bw_2 \Vert.
\end{align*}
This proves that $f|_{\mathcal{C}}$ obeys $(0,L_1)$-smooth condition.

Also, by induction, one can show that $f|_{\mathcal{C}}\ge 0$. Together, we can then interpolate $f|_{\mathcal{C}}$ to get $f$, which satisfies Assumption \ref{assum: non-smooth} and $f\ge 0$. Furthermore, running AdaGrad-Norm on $f$ with the initial point $\bw_1=(0,0)$ leads to that $\bw_{2k-1}=(S_{2k-3},S_{2k-2})$ and $\bw_{2k}=(S_{2k-1},S_{2k-2})$, and $\Vert \nabla f(\bw_{k}) \Vert= 2^{k}$. 

The proof is completed.
\end{proof}
\end{document}